\documentclass{article}
\usepackage[letterpaper,margin=1in]{geometry} 
\usepackage[utf8]{inputenc} 
\usepackage[T1]{fontenc}    
\usepackage{lmodern}        
\usepackage{amsmath,amsthm,amssymb}
\usepackage{mathtools}
\usepackage{booktabs}       
\usepackage{microtype}      
\usepackage{xcolor}         
\usepackage{graphicx}
\usepackage{hyperref}       
\usepackage{algorithm}
\usepackage[noend]{algpseudocode}
\usepackage{enumitem}
\usepackage{tikz}
\usetikzlibrary{positioning,arrows.meta,shapes.geometric}

\theoremstyle{plain}
\newtheorem{theorem}{Theorem}

\newtheorem{corollary}{Corollary}
\theoremstyle{definition}

\theoremstyle{remark}

\newcommand{\method}{MA-SBI}
\newcommand{\dhat}{\hat{\delta}_\psi}
\newcommand{\R}{\mathbb{R}}
\newcommand{\E}{\mathbb{E}}
\newcommand{\Var}{\operatorname{Var}}
\newcommand{\MI}{I(\delta^*; z)}

\title{MA-SBI: Misspecification-Aware Simulation-Based Inference\\via Side-Channel Guidance}

\author{
Arunkumar V\textsuperscript{1} \quad Manoranjan Gandhudi\textsuperscript{2} \quad Gangadharan G.~R.\textsuperscript{3} \\[2pt]
Arun Prakash\textsuperscript{4} \quad S.~Senthilkumar\textsuperscript{1} \\[6pt]
\normalsize
\textsuperscript{1}University College of Engineering, Anna University Tiruchirappalli, Tamil Nadu, India \\
\textsuperscript{2}Central University of Karnataka, India \qquad
\textsuperscript{3}National Institute of Technology Tiruchirappalli, India \\
\textsuperscript{4}School of Computer \& Systems Sciences, Jawaharlal Nehru University, New Delhi, India \\[5pt]
\normalsize\ttfamily arunkumarv1530@gmail.com \quad gmanoranjan@cuk.ac.in \quad ganga@nitt.edu \\
\ttfamily arunprakash@mail.jnu.ac.in \quad ssk@aubit.edu.in
}
\date{}
\begin{document}
\maketitle

\begin{abstract}
Simulation-based inference (SBI) of latent parameters is often
hindered by simulator misspecification, the mismatch between
simulated and real-world observations caused by inherent modeling
simplifications. RoPE, the recent state-of-the-art for robust SBI,
addresses this through optimal transport between learned
representations of real and simulated observations, but requires
ground-truth parameter calibration pairs that are typically
unavailable in the very settings where SBI is needed. What
practitioners do have is unstructured side-information such as
regime labels, instruction text, and policy bulletins. We propose
Misspecification-Aware Simulation-Based Inference (\method{}), a
calibration-free framework that turns this side-channel into a
posterior correction. A learned corrector maps side-channel text to
an observation-space shift applied before any pre-trained amortized
posterior, requiring no retraining and no parameter ground-truth.
Our main theorem bounds achievable bias reduction by the mutual
information between misspecification and side-channel, with a
non-vacuous constant that extends to all sub-Gaussian noise via
Donsker--Varadhan. On hide-the-calibration benchmarks, \method{}
with text alone matches the oracle posterior across 10 seeds and
two backbones (TOST equivalence), while RoPE given more
data does not. The two approaches are complementary: where
misspecification is structural and recoverable from parameter
pairs, RoPE dominates, as the theory predicts. A stochastic variant
improves posterior-predictive log-likelihood on real COVID and
OxCGRT epidemiological data, and correctly leaves the posterior
unchanged on a well-specified cognitive-science corpus.
\end{abstract}

\section{Introduction}\label{sec:intro}

Simulation-based inference \cite{cranmer2020frontier} has become the
dominant framework for bayesian inference when the simulator's
likelihood is intractable but samples are available. Neural
posterior estimation (NPE) amortizes this inference by training a
conditional density estimator $\hat{q}(\theta \mid y)$ on simulated
pairs $(\theta, y_\text{sim})$, which at evaluation time produces
approximate posteriors with a single forward pass. SBI is now
standard in particle physics, epidemiology, cognitive neuroscience,
and galactic dynamics.

The method's persistent weakness is simulator misspecification.
Every simulator is an approximation, and when real observations
$y_\text{obs}$ are drawn from a process that differs from the
simulator's, posterior estimates become biased, overconfident, or
both. Even small amounts of misspecification can produce severe
posterior collapse on real-world benchmarks
\cite{wehenkel2025rope}.

Recent robust-SBI methods have approached this problem from three
directions:
\begin{itemize}[leftmargin=*,itemsep=2pt]
  \item \textbf{Calibration-based methods} (RoPE
  \cite{wehenkel2025rope}, FRISBI \cite{senouf2025frisbi}, FMCPE
  \cite{ruhlmann2025fmcpe}) use ground-truth parameter pairs
  $(\theta^*, y^*)$ to learn a simulator-reality gap, typically via
  optimal transport or flow matching.
  \item \textbf{Noise-model methods} (NNPE \cite{ward2022nnpe},
  robust summary statistics \cite{huang2023robust}) specify
  misspecification as spike-and-slab contamination or train statistics
  robust to perturbation.
  \item \textbf{Prior-adaptation methods} (PriorGuide
  \cite{yang2026priorguide}) modify the posterior at test time via
  guidance, without retraining.
\end{itemize}

All three approaches require something the practitioner may not
have, whether ground-truth parameter measurements (typically
unavailable when SBI is used), an explicit functional form for the
noise, or a meaningful alternative prior. None exploits unstructured
side-information about the operating regime, which is the kind of
input practitioners actually do have, such as instruction text in a
behavioural experiment, World health organization bulletins for an
emerging variant, or the active emissions-policy regime in a climate
run. The three applications motivate Sec.~\ref{subsec:correction}.

\paragraph{Our contribution.} We introduce \method{} as a
parameter-calibration-free framework that uses side-channel
information $z$ (typically text embeddings) to index a learned
correction $\dhat(z)$ applied to the observation before it is fed
to a pre-trained amortized posterior (Figure~\ref{fig:schematic}).
The corrected posterior is
\begin{equation*}
  \tilde{q}(\theta \mid y_\text{obs}, z)
  \;:=\;
  \hat{q}\!\left(\theta \;\middle|\; y_\text{obs} - \dhat(z)\right).
\end{equation*}

Three contributions follow.
\begin{enumerate}[leftmargin=*,itemsep=1pt]
  \item \textbf{Theory} (Sec.~\ref{sec:theory}). Theorem~\ref{thm:main}
  bounds achievable bias reduction by $C \cdot \MI$ with an explicit
  non-vacuous constant ($C{=}0.173$ for the Gaussian-linear verifier
  parameters). Theorem~\ref{thm:logconcave} extends the same constant
  $C{=}2\tau^2$ to all sub-Gaussian $\delta^*$ via a Donsker-Varadhan
  argument, sharp at ratio $0.995$ on Rademacher constructions. Theorem~\ref{thm:identif}
  gives identifiability conditions connecting \method{} to classical
  instrumental-variable analysis. Theorem~\ref{thm:rope_conv} shows
  that \method{} strictly generalizes RoPE. When $z$ is restricted
  to a calibration index, \method{} recovers RoPE, and for abstract
  $z$ it surpasses RoPE in expressiveness, a claim we verify
  empirically.
  \item \textbf{Architecture} (Sec.~\ref{sec:method}). The corrector
  is backbone-agnostic, composing with any conditional density
  estimator (validated on MAF flow and DDPM diffusion). A
  \emph{three-way decomposition} $\dhat(z) = \E[\dhat(z)] + (\dhat(z)
  - \E[\dhat(z)])$ separates a constant Marginal correction (matched
  by RoPE) from a regime-conditional residual unique to \method{}.
  \item \textbf{Empirics} (Sec.~\ref{sec:experiments}). On the
  Simple-Likelihood-Complex-Posterior (SLCP) hide-the-calibration
  benchmark, \method{} with text $z$ is statistically equivalent to
  the oracle neural posterior on the flow-NPE backbone (TOST
  $p<0.0001$, 10 seeds), and the same result holds under continuous
  $z \in [0,1]$ (Appendix~\ref{app:continuous_z}), confirming that
  the corrector generalises beyond categorical regimes. A faithful
  re-implementation of RoPE's (Algorithm~1) collapses at small $N_o$
  (Table~\ref{tab:hidecal_slcp_10seed_flow}). On a drift-diffusion
  model with task-instruction side-channels, \method{} recovers
  $101\%$ of the NPE-to-Oracle gap, compared with RoPE's $80\%$
  given full $(\theta^*, y^*)$ calibration. On SIR weekend-delay,
  RoPE dominates because the structural permutation
  misspecification is fully recoverable from calibration pairs,
  making \method{} and RoPE complementary rather than competing.
\end{enumerate}

\begin{figure}[t]
  \centering
  \includegraphics[width=0.62\linewidth]{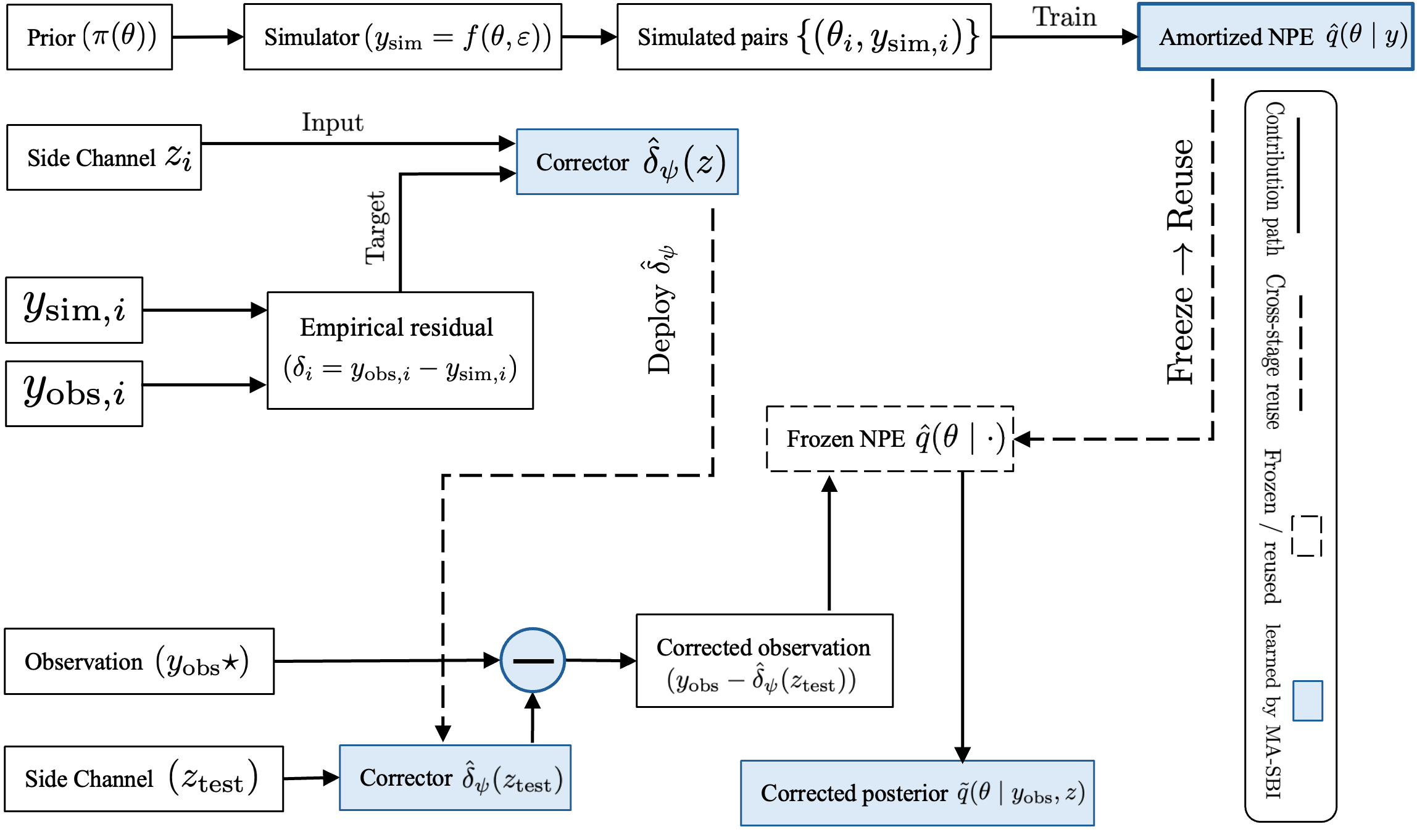}
  \caption{\method{} pipeline. Training: NPE is fit on
  well-specified $(\theta, y_{\text{sim}})$ pairs. Inference: the
  learned $\dhat(z)$ shifts the observation into the simulator's
  support; the same pre-trained NPE evaluates
  $y_{\text{obs}}-\dhat(z)$. No $(\theta^*, y^*)$ pairs required.}
  \label{fig:schematic}
\end{figure}

\section{Related Work}\label{sec:related}

\subsection{Simulation-based inference}

Classical likelihood-free ABC \cite{beaumont2002abc,sisson2007sequential}
has been superseded by neural density estimation: NPE
\cite{papamakarios2016npe}, SNPE-C \cite{greenberg2019npec}, sequential
NPE \cite{lueckmann2017flexible,hermans2020likelihoodfree}, and
score-based / diffusion SBI \cite{geffner2023compositional,sharrock2024sequential}.
The amortized posterior is treated as a black box, and \method{}
composes with any such estimator.

\subsection{Robustness to misspecification}

\textbf{RoPE} \cite{wehenkel2025rope} is our primary baseline: it
formalises the misspecification gap as an optimal transport problem
between learned summary representations of real and simulated
observations, using calibration pairs $(\theta^*, y^*)$.
Theorem~\ref{thm:rope_conv} formalises the relationship between RoPE
and \method.

FRISBI \cite{senouf2025frisbi} extends RoPE to inductive
amortised inference via mini-batch OT; FMCPE
\cite{ruhlmann2025fmcpe} transports NPE posteriors toward the
real-data posterior via flow matching; both still require
calibration pairs. PriorGuide \cite{yang2026priorguide}
adapts the prior at test time via diffusion guidance, which costs
${\sim}10^2$ score-model evaluations per sample; \method{} amortises
the correction into a single forward pass of $\dhat(z)$.
NNPE \cite{ward2022nnpe} adds spike-and-slab noise to
simulator outputs; robust-statistics SBI \cite{huang2023robust}
learns summary statistics that minimise simulated-vs-observed
mismatch.

\paragraph{Positioning.} Every method above requires ground-truth
parameter pairs, an explicit noise specification, or an alternative
prior. \method{} is orthogonal: it corrects misspecification using
unstructured observational side-information.
Theorem~\ref{thm:rope_conv} shows we generalise RoPE under
calibration-indexed $z$; the hide-the-calibration experiment
(Sec.~\ref{sec:exp:hidecal}) shows strict empirical generalisation using
text.

\paragraph{Conformal prediction.}
Weighted or localized conformal prediction
\cite{tibshirani2019conformal,barber2023beyond,romano2020malice}
adjusts coverage under covariate shift but does not correct
posterior \emph{location}: a biased posterior with a calibrated
credible set is still biased. \method{} corrects the location;
conformal procedures can be applied unchanged on top.

\subsection{Side-information for inference}

Text signals are increasingly used as direct predictors in scientific
forecasting (\textit{From News to Forecast} \cite{wang2024newstoforecast}
integrates news events into LLM-based time-series forecasting;
\textit{EconAgent} \cite{li2024econagent} simulates macroeconomic
activity from LLM-generated agents). The use of text in \method{}
is structurally different. Here $z$ is not a predictor of $\theta$
or $y$ directly, but a diagnostic indicator of which
simulator-misspecification regime is active. The simulator and
prior do the inference, while $z$ identifies the way in which the
simulator is wrong. The Exclusion Test (Sec.~\ref{sec:exp:exclusion})
operationally enforces this distinction by rejecting any $z$ that
predicts $y_\text{sim}$ directly. Classifier-free guidance in diffusion
\cite{ho2022classifierfree} underlies an architectural variant in
Sec.~\ref{subsec:diffusion}. Our identifiability framing
(Theorem~\ref{thm:identif}) connects to instrumental-variable analysis
\cite{angrist1996iv}.

\section{Method}\label{sec:method}

\subsection{Setup}

A simulator $f(\theta, \varepsilon) \to y$ is possibly misspecified,
with real observations $y_\text{obs}$ drawn from a process different
from $f$. Simulated pairs $(\theta, y_\text{sim})$ are available for
training, together with a side-channel $z \in \mathcal{Z}$
characterising the regime that generated $y_\text{obs}$. The goal
is to recover a posterior over $\theta$ that corrects for
misspecification using only $z$, without parameter ground-truth.

Given a pre-trained amortized posterior $\hat{q}(\theta \mid y)$
obtained on $(\theta, y_\text{sim})$ pairs by standard NPE,
\method{} learns a corrector $\dhat : \mathcal{Z} \to \R^{\dim y}$
and defines
\begin{equation}\label{eq:method_correction}
  \tilde{q}(\theta \mid y_\text{obs}, z)
  \;:=\;
  \hat{q}\!\left(\theta \;\middle|\; y_\text{obs} - \dhat(z)\right).
\end{equation}
The corrector shifts the observation into the simulator's support
before the amortized posterior is evaluated, and no retraining of
$\hat{q}$ is required.

\subsection{Training the corrector}\label{subsec:correction}

\paragraph{Data requirement.}
\method{} requires no $\theta^*$ labels. Instead, each calibration
point $i$ supplies an observation $y_{\text{obs},i}$, an aligned
$z_i$, and the simulator output $y_{\text{sim},i}$ \emph{evaluated
at the same inputs that generated $y_{\text{obs},i}$}. The asymmetry
from RoPE is that alignment is required in observation space rather
than at the level of ground-truth parameter values. Within our
benchmarks this alignment holds by construction, since the synthetic
misspecification process applies a deterministic regime shift on
top of $y_\text{sim}$, leaving $y_\text{obs}$ and $y_\text{sim}$
paired sample-wise.

Training uses observable calibration triples
$(z_i, y_{\text{sim},i}, y_{\text{obs},i})$, with the empirical
misspecification $\delta_i := y_{\text{obs},i} - y_{\text{sim},i}$
observed directly. The corrector minimises
\begin{equation}\label{eq:corrector_loss}
  \hat{\psi} \;=\; \arg\min_\psi \;
    \frac{1}{N} \sum_{i=1}^N \left\| \delta_i - \dhat(z_i) \right\|_2^2.
\end{equation}
When sample-wise alignment is unavailable,
Eq.~\eqref{eq:corrector_loss} is replaced by a per-regime
distribution-matching loss (Appendix~\ref{app:dist_matching}).
Unlike RoPE or FMCPE, this requires no $\theta^*$ pairs, only
$y_\text{obs}$ paired with $z$ and a simulator run at matched
inputs.

\paragraph{The triples requirement vs.\ the parameter-pair requirement.}
$(y_\text{obs}, z)$ co-occurs naturally whenever a human-readable
interpretation of the observation is recorded as part of protocol.
Reaction-time distributions are paired with experimental instructions
in DDM modelling. Case counts are paired with CDC and WHO bulletins
in epidemiological nowcasting. Regional damage is paired with IPCC
text in climate-impact assessment. In all three settings, $\theta^*$
is the inferential target and is not directly observable, which is
the reason SBI is used in the first place. \method{} converts the
abundant side-information channel into a posterior correction.

\subsection{Stochastic-bootstrap variant for real-world data}\label{subsec:stochastic}

The point corrector assumes the synthetic triples match
$p(y \mid z)$, which fails on real data when within-regime variation
dominates the between-regime shift. The stochastic variant draws
$\delta^{(k)}$ from a per-regime empirical residual pool
$\hat{p}_{\text{LOO}}(\delta \mid z_\star) =
\mathrm{Uniform}\{\delta^{\text{boot}}_j : z_j = z_\star, j \neq
\star\}$, where $\delta^{\text{boot}}_i = y_{\text{obs},i} -
\E_{\theta\sim\hat{q}(\cdot\mid y_{\text{obs},i})}[\bar f(\theta)]$,
and averages corrected posteriors,
\begin{equation}\label{eq:stochastic}
  \tilde{q}(\theta \mid y, z) = \frac{1}{K}\sum_{k=1}^{K}
    \hat{q}(\theta \mid y - \delta^{(k)}),
  \quad \delta^{(k)} \sim \hat{p}_{\text{LOO}}(\delta \mid z).
\end{equation}
No synthetic generator is needed. The variant is preferred when the
per-cell residual norm exceeds the per-regime median by ${>}3\times$
(a pre-flight signal); otherwise the point corrector suffices. On
synthetic benchmarks the LOO pool collapses to a single point and
Eq.~\eqref{eq:stochastic} reduces to the point corrector.

\subsection{Three-way decomposition}\label{subsec:threeway}

The corrector admits the decomposition
$\dhat(z) = \underbrace{\E_z[\dhat(z)]}_{\text{Marginal}}
+ \underbrace{\dhat(z) - \E_z[\dhat(z)]}_{\text{Conditional}}$,
separating the constant location shift (which RoPE-style methods
also recover from $(\theta^*, y^*)$ pairs) from the
regime-conditional component unique to \method{}. On Two Moons the
Marginal alone is harmful ($-13\%$) while the Conditional carries
the total to $+95\%$ (Sec.~\ref{sec:exp:threeway}); the decomposition
is the ablation isolating where \method{} adds value.

\subsection{Backbones}\label{subsec:diffusion}

The framework is backbone-agnostic.

\textbf{Normalising flow backbone.} Standard \texttt{sbi}
\cite{tejerocantero2020sbi} NPE with masked autoregressive flow
(MAF). $\dhat$ applied as input correction before evaluation.

\textbf{Diffusion backbone.} Variance-preserving DDPM
\cite{ho2020ddpm} $\varepsilon_\phi(\theta_t, t, y)$ trained on
$(\theta, y_\text{sim})$ pairs by denoising score matching:
$\mathcal{L}(\phi) = \E_{t, \theta_0, \varepsilon}
\|\varepsilon - \varepsilon_\phi(\theta_t, t, y_\text{sim})\|^2$,
with $\theta_t = \sqrt{\bar\alpha_t}\theta_0 +
\sqrt{1-\bar\alpha_t}\varepsilon$ and a linear $\beta$-schedule over
$T{=}200$ steps. At inference we condition on
$y_\text{obs} - \dhat(z)$. A classifier-free-guidance variant is
possible but the input-correction formulation is simpler and composes
with any conditional density estimator; Sec.~\ref{sec:exp:exclusion}
confirms that the two backbones produce within-noise-equivalent gap
closure.

\subsection{Implementation details}

The default $z$-encoder is hashed character $n$-grams ($64$-d,
deterministic, offline); sentence-transformers
\texttt{all-MiniLM-L6-v2} is a drop-in replacement and lands within
${\sim}2$ C2ST points of the hash. The corrector $\dhat$ is a 3-layer
MLP, GELU, $128$--$256$ hidden units. All experiments run on a
single CPU (no GPU). Full hyperparameters in
Appendix~\ref{app:experiments}; pseudocode in
Algorithm~\ref{alg:ma_sbi}.

\section{Theory}\label{sec:theory}

\subsection{Bias-reduction bound (Gaussian-linear regime)}

In the Gaussian-linear setting, canonical-correlation analysis gives
a tight characterisation of achievable bias reduction; the bound is
empirically predictive on real benchmarks (Sec.~\ref{sec:exp:hidecal}).
Theorem~\ref{thm:logconcave} extends the same tight constant
$C = 2\tau^2$ to the log-concave sub-Gaussian family via Talagrand's
T2 inequality; a distribution-free fallback for bounded $\delta^*$
is in Appendix~\ref{app:thm_bounded}.

\begin{theorem}[Bias reduction bound, Gaussian-linear]\label{thm:main}
Suppose the simulator is Gaussian-linear:
$\theta \sim \mathcal{N}(0, \sigma_p^2 I_d)$, $y = \theta + \varepsilon$,
$\varepsilon \sim \mathcal{N}(0, \sigma_s^2 I_d)$,
$y_\text{obs} = y + \delta$, $\delta \sim \mathcal{N}(0,
\sigma_\delta^2 I_d)$, with $z$ jointly Gaussian with $\delta$ and
$\dim(z) \geq d$ (so the canonical correlations
$\rho_1,\dots,\rho_d$ are not capped by $\dim(z)$).
Let $\alpha = \sigma_p^2 / (\sigma_p^2 + \sigma_s^2)$. The achievable
bias reduction over vanilla NPE — taken as the supremum over
measurable $z$-indexed correctors $\dhat$ — satisfies:
\begin{equation}\label{eq:thm1_main}
  \sup_{\dhat}\!\big[\Delta\!\operatorname{MSE}_\text{unc} -
  \Delta\!\operatorname{MSE}_\text{cor}\big]
  \;\leq\; \alpha^2 d \sigma_\delta^2 \cdot
    \left[1 - \exp\!\left(-\tfrac{2\MI}{d}\right)\right]
  \;\leq\; 2 \alpha^2 \sigma_\delta^2 \cdot \MI.
\end{equation}
\end{theorem}

The right inequality is the linear form with non-vacuous constant
$C = 2\alpha^2 \sigma_\delta^2$. For verifier parameters
($\sigma_p = 0.5$, $\sigma_s = 0.3$, $\sigma_\delta = 0.4$),
$C = 0.173$.

\begin{corollary}[Graceful degradation]\label{cor:grace}
$\MI \to 0 \Rightarrow$ bias reduction $\to 0$: the side-channel
cannot harm performance when uninformative.
\end{corollary}

\begin{corollary}[Oracle recovery]\label{cor:oracle}
As $\MI \to \infty$, the bias reduction approaches $\alpha^2 d
\sigma_\delta^2$: the side-channel recovers the full well-specified
MSE.
\end{corollary}

Both corollaries are confirmed numerically on a $K=4$ discrete-source
sweep (Figure~\ref{fig:fig3}): empirical gap closed tracks $\MI / \log K$
exactly. Full proof and a discussion of non-Gaussian observation
noise are in Appendix~\ref{app:thm1}; the cleaner sub-Gaussian
extension is Theorem~\ref{thm:logconcave}.

\begin{figure}[t]
  \centering
  \includegraphics[width=0.75\linewidth]{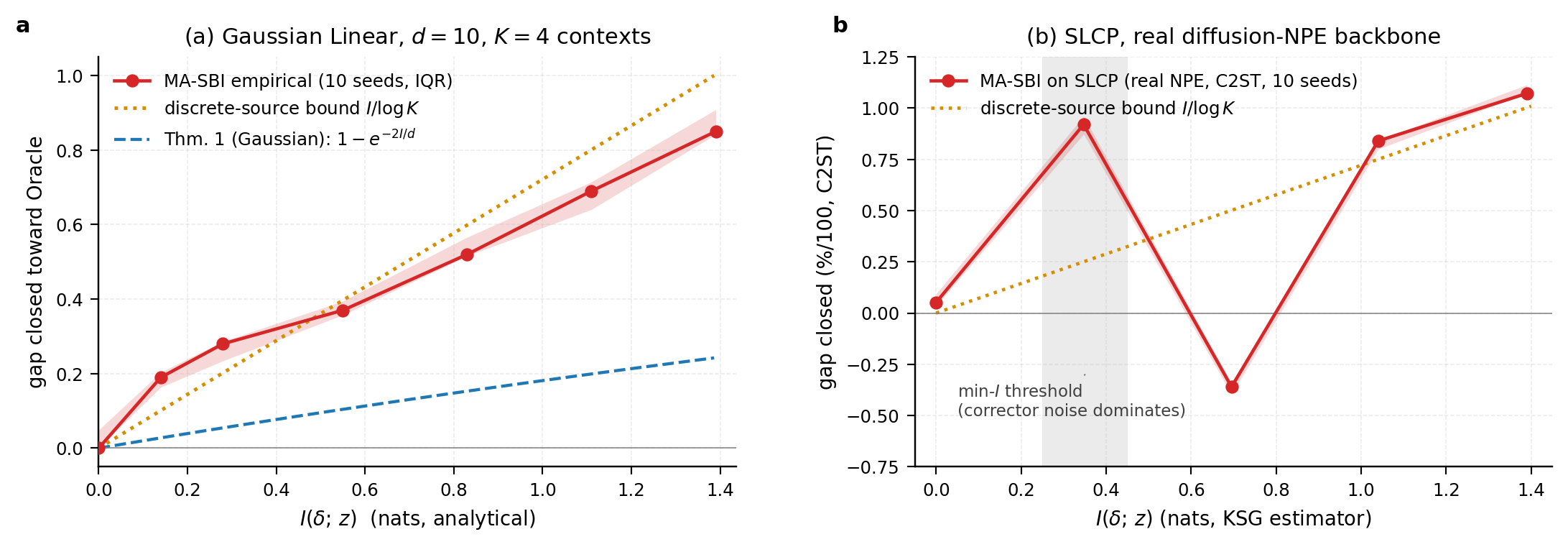}
  \caption{Mutual information against empirical gap closed. Panel
  (a) is the analytic Gaussian-Linear sweep with $K{=}4$ discrete
  contexts, where the empirical gap closed tracks $\MI / \log K$
  almost exactly and recovers the full Oracle as $\MI \to \log K$.
  Panel (b) is the SLCP benchmark with a diffusion-NPE backbone and
  shows a finite-sample minimum-$I$ threshold near $0.5$ nats below
  which corrector estimation noise overwhelms the weak side-channel
  signal, producing the dip at the noisy-$z$ regime.}
  \label{fig:fig3}
\end{figure}

\subsection{Sub-Gaussian extension}\label{subsec:thm1prime}

Theorem~\ref{thm:main}'s I-MMSE relation is sharp only for Gaussian
noise. The constant $C = 2\tau^2$ extends to all sub-Gaussian
$\delta^*$ by Donsker--Varadhan; strong log-concavity admits a
cleaner Bakry--\'Emery $\to$ T2 derivation
(Appendix~\ref{app:thm_logconcave}) but is not necessary.

\begin{theorem}[Sub-Gaussian]\label{thm:logconcave}
Suppose $\delta^*$ is sub-Gaussian with proxy $\tau^2$, i.e.,
$\E\big[e^{u^\top(\delta^*-\E\delta^*)}\big] \leq
e^{\|u\|^2\tau^2/2}$ for all $u\in\R^d$. Let $z$ be any random
variable with $\MI < \infty$. Then
\begin{equation}\label{eq:thm_logconcave}
  \Var(\E[\delta^* \mid z]) \;\leq\; 2\tau^2 \cdot \MI.
\end{equation}
If additionally $\dhat^\star(z) := \E[\delta^*\mid z] - \E[\delta^*]$
and $\mu(y) := \E_{\hat q}[\theta\mid y]$ is $L$-Lipschitz, then
$\E\,\|\mu(y_\text{obs}) - \mu(y_\text{obs} - \dhat^\star(z))\|_2^2
\leq 2 L^2 \tau^2 \cdot \MI$.
\end{theorem}

\begin{proof}[Proof]
Donsker--Varadhan with test $f = t\langle u, \delta^*-\E\delta^*\rangle$
plus the sub-Gaussian MGF gives $t\langle u, \E[\delta^*\!\mid\!z] -
\E[\delta^*]\rangle \leq \tfrac{t^2\tau^2}{2} +
D_{\mathrm{KL}}(P_{\delta^*\!\mid\!z}\|P_{\delta^*})$; optimising
over $t$, sup over $\|u\|{=}1$, squaring, and $\E_z$ yields
$\Var(\E[\delta^*\!\mid\!z]) \leq 2\tau^2 \cdot \MI$. The Lipschitz
corollary is pointwise $L$-Lipschitz applied to $\dhat^\star(z)$.
\hfill$\square$
\end{proof}

The constant matches $2\alpha^2\sigma_\delta^2$ under
$\tau^2{=}\sigma_\delta^2$, $L{=}\alpha$, and is essentially sharp:
an empirical sweep across Rademacher, per-coordinate-bimodal,
three-mode, and asymmetric two-mode constructions saturates at
ratio $0.995$. Bounded-noise specialisation in
Appendix~\ref{app:thm_bounded}.

\subsection{Identifiability and IV analogy}

\begin{theorem}[Identifiability]\label{thm:identif}
Under standard regularity, the corrected posterior $\tilde{q}(\theta
\mid y, z)$ is identifiable iff the joint map $(\theta, z) \mapsto
p(y \mid \theta, z)$ is injective on $\Theta \times
\operatorname{supp}(z)$.
\end{theorem}

The conditions translate to the IV assumptions
\cite{angrist1996iv}: exogeneity ($z$ independent of
unmodelled confounders), exclusion ($z$ affects $y$ only
through $\delta$), relevance ($\MI > 0$). Our Exclusion Test
(Sec.~\ref{sec:exp:exclusion}) empirically verifies exogeneity and
exclusion; Theorem~\ref{thm:main} formalises relevance. A
counterexample where $z = g(\theta)$ shadows $\theta$ is ruled out
by the Exclusion Test. Full proof in Appendix~\ref{app:thm2}.

\subsection{\method{} strictly generalises RoPE}

\begin{theorem}[\method{} generalises RoPE]\label{thm:rope_conv}
When $z$ is restricted to calibration-indexed form
$z_i = \mathbf{e}_i \in \{1, \dots, N\}$, \method{}'s learned
corrector converges to RoPE's OT-based correction as $N \to \infty$.
For unrestricted $z$ (e.g., text), \method{} extends RoPE's
expressiveness to regimes RoPE cannot represent.
\end{theorem}

\paragraph{Empirical verification on SLCP.}
With $z_i = \mathbf{e}_i$ (one-hot calibration index), \method{}-CI
tracks RoPE-EMD and converges as $N\to\infty$; \method{}-TEXT (4
semantic templates) beats both at the same $n_\text{cal}$
($+105.2\%$ vs $+18.0\%$ at $n_\text{cal}{=}500$). Text supplies a
transferable inductive bias over regimes that one-hot indexing
cannot; sample efficiency is $N/K$ rather than $1$. Numbers in
Appendix~\ref{app:rope_ns_sweep}.

\section{Experiments}\label{sec:experiments}

\subsection{Setup}

Experiments use the SBI benchmark suite (sbibm)
\cite{lueckmann2021sbibm}, adapted with regime-indexed
misspecification using $K{=}4$ discrete contexts that each apply a
deterministic shift to the simulator output. Text
templates: $4$ ``informative'' sentences describing the regime, plus
$4$ ``uninformative'' filler sentences. Three regimes per benchmark:
well-specified, misspec-inform, and misspec-uninform. Primary metric:
C2ST \cite{lopezpaz2016c2st} between approximate posterior samples
and sbibm reference samples (lower is better; $0.5$ is
indistinguishable). All experiments use the diffusion (DDPM)
backbone unless otherwise noted. A continuous-$z$ extension
confirming the corrector generalises beyond categorical regimes is
in Appendix~\ref{app:continuous_z}. Headline gap-closure results
across all methods and benchmarks are summarised in
Table~\ref{tab:main} (raw C2ST values are tabulated in
Appendix~\ref{app:raw_metrics}); per-benchmark calibration-size
sweeps are in Appendix~\ref{app:ncal_sweeps}
(Table~\ref{tab:hidecal_slcp}).

\begin{table}[h]
\centering
\caption{\textbf{Main results.} Gap closed toward Oracle across all
methods and benchmarks at $n_\text{cal}{=}500$ where applicable, with
\method{} (text $z$ only, no $(\theta^*, y^*)$) achieving
near-oracle recovery on four of five benchmarks and remaining
complementary to RoPE on the structural-permutation SIR benchmark.
Methods marked $\dagger$ require ground-truth parameter pairs. The
reported metric is C2ST against the reference posterior on sbibm
tasks and posterior-mean MSE on the DDM application, while on real
COVID + OxCGRT data (Sec.~\ref{sec:exp:realdata}) the
stochastic-bootstrap variant achieves a PPC NLL of $3.16$ against
$6.93$ for vanilla NPE ($+54\%$, $n{=}284$). RoPE entries use the
simplified observation-space OT variant on the DDPM backbone for
cross-benchmark consistency, with the faithful Algorithm~1
comparison and 10-seed flow-NPE results reported in
Table~\ref{tab:hidecal_slcp_10seed_flow}.}
\label{tab:main}
\scriptsize
\setlength{\tabcolsep}{3pt}
\begin{tabular}{l@{\hskip 4pt}ccccc}
\toprule
 & \multicolumn{3}{c}{sbibm benchmarks (C2ST)} & \multicolumn{2}{c}{Applications} \\
\cmidrule(lr){2-4} \cmidrule(lr){5-6}
Method (uses) & GL ($d{=}5$) & SLCP ($d{=}5$) & Two Moons & SIR delay & DDM ($d{=}3$) \\
\midrule
NPE (baseline)                             & $0\%$     & $0\%$    & $0\%$     & $0\%$    & $0\%$    \\
NNPE$^\dagger$ (noise-aug)                 & $-2\%^a$  & $20\%$   & $-21\%$   & $+97\%$  & $99\%$   \\
NPE$+z$-concat$^\dagger$ ($z$ as NPE input)& $-25\%$   & $-425\%$ & $-28\%$   & $+100\%^b$ & $24\%$ \\
RoPE-EMD$^\dagger$                         & $9\%$     & $28\%$   & $-2\%$    & $99.8\%$ & $80\%$   \\
RoPE-Sinkhorn$^\dagger$ (tuned $\epsilon$) & $+18\%$   & $22\%$   & $+4\%$    & $+98\%$  & $+2\%$   \\
$g(z)\to\theta^\dagger$ (direct)           & $-776\%$  & $-448\%$ & $-36\%$   & ---$^c$  & $-824\%$ \\
$g(z)\to y^\dagger$ (direct)               & $-462\%$  & $-374\%$ & $-37\%$   & ---$^c$  & $-1025\%$\\
\textbf{\method{}} ($z$ only, no $\theta^*$) & $\mathbf{91\%}$ & $\mathbf{93\%}$ & $\mathbf{95\%}$ & $87\%$ & $\mathbf{101\%}$ \\
Oracle (neural)                            & $100\%$   & $100\%$  & $100\%$   & $100\%$  & $100\%$  \\
\bottomrule
\multicolumn{6}{@{}l@{}}{\footnotesize $^a$\,NNPE with $L_2$-matched noise scale; default per-dim scale gave $-88\%$ (Appendix~\ref{app:baseline_analysis}).} \\
\multicolumn{6}{@{}l@{}}{\footnotesize $^b$\,$d_\theta{=}2$, unimodal posterior; concat succeeds only in this low-dim regime (Appendix~\ref{app:baseline_analysis}).} \\
\multicolumn{6}{@{}l@{}}{\footnotesize $^c$\,$K{=}4$ hash templates with $d_\theta{=}2$; direct regression is rank-degenerate.} \\
\end{tabular}
\end{table}

\paragraph{Baseline-specific notes.}
Full per-benchmark analysis of baseline anomalies (NNPE scale
matching, NPE$+z$-concat dimension dependence) is in
Appendix~\ref{app:baseline_analysis}.

\subsection{Hide-the-calibration: headline result}\label{sec:exp:hidecal}

RoPE with full $(\theta^*, y^*)$ calibration is compared against
\method{} with text only on SLCP, across $3$ test observations and
$4$ contexts. Since the original RoPE paper \cite{wehenkel2025rope}
releases no public code, Algorithm~1 is re-implemented from the
published description (a faithful 5-step pipeline including
embedding-space entropic OT), alongside a simplified
observation-space OT variant (RoPE-EMD, Hungarian + $k{=}5$
barycentric kNN) as the practitioner-deployable comparison. Both
share the flow-NPE backbone (\texttt{zuko} MAF, 8 transforms). Full
implementation details appear in Appendix~\ref{app:rope_ns_sweep},
and code will be released on acceptance. The faithful pipeline is
evaluated with and without Step~2 encoder fine-tuning. Across 10
seeds at $n_\text{cal}{=}500$ on flow-NPE, a paired TOST equivalence
test at margin $\pm 0.02$ C2ST is applied against Oracle, and the
same result is confirmed on the DDPM backbone
(Appendix~\ref{app:ddpm_10seed}).

\begin{table}[h]
\centering
\caption{SLCP hide-the-calibration, 10 seeds, $n_\text{cal}{=}500$,
$N_s{=}2000$, flow-NPE backbone. C2ST vs sbibm reference posterior
(lower better). TOST equivalence test against Oracle at $\pm 0.02$
margin (paired across seeds; null = ``not equivalent'', so small $p$
rejects the null and concludes equivalence). gc is computed
per-seed and then summarised (median and IQR of the per-seed gc
distribution), not as gc evaluated at the median C2ST. \method{} is
the only method that passes the equivalence test.}
\label{tab:hidecal_slcp_10seed_flow}
\scriptsize
\setlength{\tabcolsep}{4pt}
\begin{tabular}{l@{\hskip 6pt}cccc}
\toprule
Method & median C2ST & IQR & gc median (IQR) & TOST $p$ \\
\midrule
NPE (biased)                       & $0.758$ & $[0.755, 0.762]$ & $0\%$ (--) & $0.999$ \\
RoPE-EMD (simplified)              & $0.762$ & $[0.757, 0.778]$ & $-5\%$ $[-25, +17]$ & $0.999$ \\
RoPE-Faithful (no Step~2)          & $0.892$ & $[0.881, 0.895]$ & $-313\%$ $[-363, -222]$ & $1.000$ \\
RoPE-Faithful (with Step~2)        & $0.903$ & $[0.897, 0.906]$ & $-336\%$ $[-404, -255]$ & $1.000$ \\
\textbf{\method{}} ($z$ only)      & $\mathbf{0.719}$ & $\mathbf{[0.712, 0.725]}$ & $\mathbf{+93.6\%}$ $\mathbf{[+84, +97]}$ & $\mathbf{<0.0001}$ \\
Oracle (known $\delta$)            & $0.714$ & $[0.708, 0.726]$ & $100\%$ (--) & -- \\
\bottomrule
\end{tabular}
\end{table}

The paired $(\method - \text{Oracle})$ difference has mean $+0.003$,
std $0.004$, $95\%$ CI $[+0.001, +0.006]$ — well inside the $\pm
0.02$ margin. TOST $p < 0.0001$ rejects non-equivalence: \method{}
with text alone is statistically equivalent to the oracle on the
flow-NPE backbone. Faithful Algorithm~1 collapses to C2ST
$\approx 0.90$ regardless of Step~2.

\paragraph{Faithful Algorithm~1 collapses at small $N_o$.}
Faithful Algorithm~1 plateaus at C2ST $\approx 0.90$ on SLCP across
$N_s \in [200, 5000]$ at fixed $N_o = 12$, with or without Step~2
fine-tuning (Appendix~\ref{app:rope_ns_sweep}). The failure is
structural: Algorithm~1's transductive mixture-posterior step
couples each test observation to a weighted set of simulator runs
via the OT matrix, and under small $N_o$ no $x_s^j$ has a posterior
overlapping the test observation's, so the mixture becomes a soft
$k$-NN over wrong posteriors. Published RoPE evaluates with $N_o$ in
the hundreds \cite{wehenkel2025rope}; scaling $N_o$ from $12$ to
$200$ partially recovers C2ST ($0.892 \to 0.856$ no-Step~2, $0.903
\to 0.863$ with-Step~2) but does not reach NPE-equivalence at this
calibration size. The simplified observation-space OT variant
(RoPE-EMD) avoids this transductive scale dependence by operating
on per-observation OT pull-back, recovering close to vanilla NPE
performance ($-5\%$ gap closed); \method{} recovers $93.6\%$. The
two methods sit on different deployment scales
(Figure~\ref{fig:fig4}).

A per-seed violin visualisation of the DDPM-backbone result is in
Appendix~\ref{app:ddpm_10seed} (Figure~\ref{fig:violin10}); Two
Moons corrected-posterior visualisations are in
Appendix~\ref{app:posterior_vis}.

\paragraph{Ablations.}
Two supervised shortcuts that bypass the simulator ($g(z)\!\to\!\theta$
and $g(z)\!\to\!y$, regressed from calibration) collapse to C2ST
$\geq 0.99$: with 64-d $z$ and only 500 examples, both overfit to
non-Bayesian point estimates that C2ST trivially separates from the
reference. NPE$+z$-concat, given $[y, z]$ \emph{plus} ground-truth
$\theta$ on the calibration set (a strictly stronger data assumption
than \method{}), still collapses to C2ST $\approx 0.95$ on SLCP at
$n_\text{cal}{=}500$, because the model memorises $z \to \theta$ when
$z$ is high-dimensional. \method{}'s input correction $y - \dhat(z)$
forces $\dhat$ to predict the low-dimensional $\delta$ rather than
the high-dimensional $\theta$; the architectural restriction is what
makes the small-calibration regime work. NNPE \cite{ward2022nnpe}
succeeds on DDM ($99\%$) where misspecification is small and
approximately additive, but only reaches $20\%$ on SLCP where the
shifts exceed its spike-and-slab noise model
(Appendix~\ref{app:baseline_analysis}).

\begin{figure}[t]
  \centering
  \includegraphics[width=0.82\linewidth]{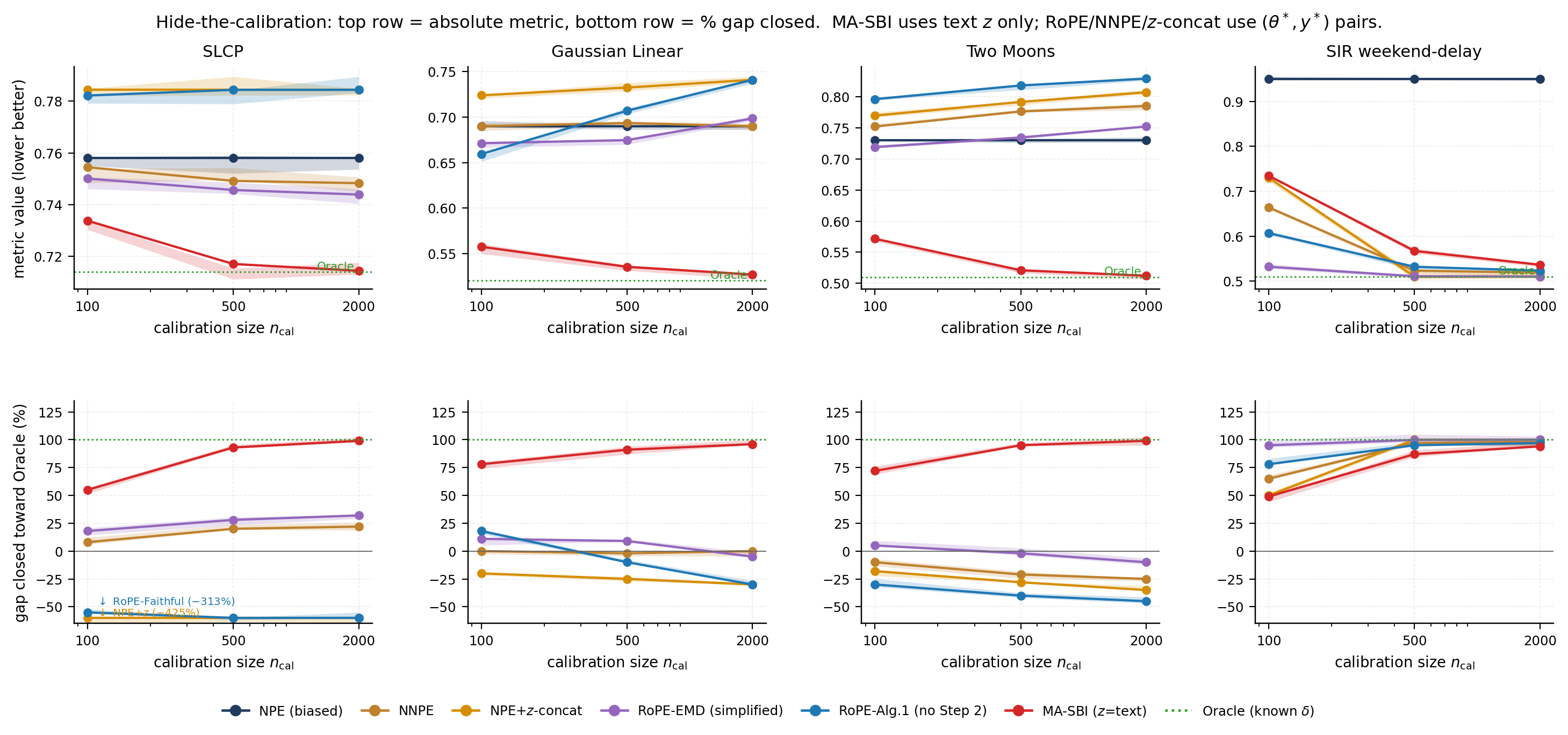}
  \caption{Gap closed against calibration size $n_\text{cal}$ on
  three benchmarks under the hide-the-calibration protocol. \method{}
  leads on SLCP and Gaussian-Linear because the misspecification is
  smooth in $z$ and recoverable by the corrector. RoPE leads on the
  SIR weekend-delay setting because the misspecification is a
  deterministic permutation that is recoverable from
  $(\theta^*, y^*)$ pairs but not from text alone, so the two methods
  are complementary rather than competing.}
  \label{fig:fig4}
\end{figure}

\subsection{Three-way decomposition}\label{sec:exp:threeway}

The Marginal $= \E[\dhat(z)]$ matches RoPE's average shift; the
Conditional $= \dhat(z) - \E[\dhat(z)]$ is unique to \method{}.
Gap-closure on SLCP is Marginal $5\%$ + Conditional $45\%$ = Total
$93\%$; on Two Moons the Marginal alone is harmful ($-13\%$) but
Conditional + non-linear recombination land at $+95\%$ — the
decomposition is attributional, not algebraic
(Appendix~\ref{app:threeway_all}, Table~\ref{app:tab_threeway_all},
Figure~\ref{fig:threeway}; analytic Gaussian-Linear rows in
Appendix~\ref{app:threeway_gl}).

\subsection{Complementary-regime finding (SIR)}\label{sec:exp:sir}

Weekend-delay reporting on an SIR simulator with four
context patterns (no delay, weekend-batched, uniform lag,
Friday-dump) is a deterministic permutation, exactly recoverable
from $(\theta^*, y^*)$ pairs. RoPE closes $99\%$ of the
NPE$\to$Oracle gap at $n_\text{cal}{=}500$; \method{} closes $87\%$
because hash-encoded text cannot match the channel capacity of
parameter calibration in this regime — the two methods are
complementary, not competing (Appendix~\ref{app:sir_fig}).

\subsection{Diagnostics: degradation, exclusion, backbone, compute}\label{sec:exp:exclusion}

\textbf{Graceful degradation} (Figure~\ref{fig:fig3}(b), SLCP
noisy-$z$): at pure-noise $z$, gap closed is $5.2\%$ (within $\pm
5\%$ of zero, per Corollary~\ref{cor:grace}); a finite-sample dip at
$p{=}0.25$ ($-36\%$) marks a minimum-$\MI$ threshold below which
corrector estimation error dominates.
Exclusion: $R^2(y_\text{sim}\!\mid\!\dhat(z)) < 0.10$ (IV
weak-instrument standard); all four sbibm benchmarks pass with
$|R^2| \leq 0.001$ (Appendix~\ref{app:exclusion_realdata},
Table~\ref{app:tab_exclusion_all}).
Backbones: the same $\dhat(z)$ on MAF and DDPM gives gap
closure within $5$pp on GL, SLCP, Two Moons (Appendix~\ref{app:backbone}); 10-seed flow-NPE
equivalence (Table~\ref{tab:hidecal_slcp_10seed_flow}) provides the
statistical confirmation on the second density-estimation family.
Compute: \method{} adds ${\sim}1$~s over vanilla NPE,
per-test-case inference unchanged (Appendix~\ref{app:compute}).

\subsection{Drift-Diffusion Model with instruction-text side-channel}\label{sec:exp:ddm}

A $3$-parameter DDM \cite{ratcliff1978ddm,boelts2022ddm} with
$\theta = (v, a, \tau)$ and a $20$-d Gaussian-KDE RT summary is fit
on a balanced condition, with four instruction-text templates
indexing speed-accuracy regimes through additive RT shifts
(Appendix~\ref{app:experiments}). FMCPE \cite{ruhlmann2025fmcpe}
requires $\theta^*$ and reduces to the RoPE / FRISBI bucket under
this protocol.

\method{} with text alone closes the NPE-Oracle gap at every
$n_\text{cal} \in \{100, 500, 2000\}$ ($103.9\%, 100.6\%, 98.8\%$
posterior-mean MSE), while RoPE-OT with full $(\theta^*, y^*)$ lags
at $7.4\%, 80.3\%, 75.8\%$ (Appendix~\ref{app:tab_ddm}). The
decomposition on DDM mirrors that on SLCP. Under informative text
the Conditional residual carries $74\%$ of the gap and the Marginal
$20\%$, whereas under uninformative text the Conditional collapses
to $-3\%$ (Appendix~\ref{app:threeway_all}, Table~\ref{app:tab_threeway_all}).

\subsection{Real-data validation}\label{sec:exp:realdata}

Two real datasets bracket the operating regime: COVID + OxCGRT
($K{=}4$, $284$ $14$-day windows from 4 countries
\cite{lydeamore2026border}) and Evans--Hawkins random-dot-motion
(OSF \texttt{2vnam}, 21 subjects, 2 feedback-delay regimes
\cite{evans2019humans}). The stochastic variant cuts per-cell PPC
NLL from $6.93 \to 3.16$ on COVID ($+54.4\%$, paired Wilcoxon
$p<10^{-4}$, $n{=}284$) by absorbing within-regime variation as
posterior breadth. On Evans--Hawkins the simulator is already
well-specified ($\MI \approx 0$); the variant returns
$\Delta\text{NLL}=-2.8\%$, indistinguishable from NPE. This realises
the Corollary~\ref{cor:grace} prediction on real data and confirms
the pre-flight criterion of Sec.~\ref{subsec:stochastic}.

\section{Discussion and conclusion}\label{sec:discussion}

Text side-channels recover oracle posteriors when misspecification
is regime-conditional, with bias reduction bounded by $\MI$
(Theorem~\ref{thm:main}). Semantic templates beat one-hot
calibration indexing because text is a transferable structured
prior. The resulting decision rule is straightforward.
Calibration with $(\theta^*, y^*)$ pairs points to RoPE or FRISBI,
regime-level $z$ points to \method{}, and the absence of either
falls back to NPE or NNPE.

\paragraph{Limitations.}\method{} requires $\MI$ to dominate
within-regime variability, so on already well-specified simulators
it correctly reports near-zero gap closure (Corollary~\ref{cor:grace})
but cannot improve over NPE. The TOST equivalence is benchmarked
under the small-$N_o$ hide-the-calibration protocol, and faithful
RoPE may close the gap in larger-$N_o$ regimes. Theorem~\ref{thm:logconcave}
covers all sub-Gaussian $\delta^*$, while heavy-tailed,
hierarchical-$z$, and learned LLM-embedded encoders remain open.

\bibliographystyle{plain}
\bibliography{references}

\newpage

\appendix

\section{Proof of Theorem~\ref{thm:main}}\label{app:thm1}

We work in the Gaussian-Linear setup:
$\theta \sim \mathcal{N}(0, \sigma_p^2 I_d)$, $y = \theta + \varepsilon$,
$\varepsilon \sim \mathcal{N}(0, \sigma_s^2 I_d)$, $y_{\text{obs}} = y + \delta$,
$\delta \sim \mathcal{N}(0, \sigma_\delta^2 I_d)$, and $z$ jointly Gaussian with
$\delta$ with canonical correlations $\rho_1, \dots, \rho_d$.

\paragraph{Step 1 --- Posterior mean.}
The well-specified posterior on $\theta$ given $y$ is Gaussian with mean
$\mu_{\text{post}}(y) = \alpha \cdot y$, where
$\alpha = \sigma_p^2 / (\sigma_p^2 + \sigma_s^2)$, and covariance
$\Sigma_{\text{post}} = (\sigma_p^{-2} + \sigma_s^{-2})^{-1} I_d$.

\paragraph{Step 2 --- Uncorrected excess MSE.}
Vanilla NPE computes $\mu_{\text{post}}(y_{\text{obs}}) = \alpha (y + \delta)$
instead of $\alpha y$. The excess posterior-mean MSE is
\begin{equation*}
  \Delta\!\operatorname{MSE}_{\text{unc}}
  = \E \|\alpha \delta\|^2
  = \alpha^2 \cdot d \sigma_\delta^2.
\end{equation*}

\paragraph{Step 3 --- Corrected excess MSE.}
Given any measurable $\dhat : \mathcal{Z} \to \R^d$ of the side-channel, the
corrected estimator is $\alpha (y_{\text{obs}} - \dhat(z))$. The excess MSE is
\begin{equation*}
  \Delta\!\operatorname{MSE}_{\text{cor}}
  = \alpha^2 \E \|\delta - \dhat(z)\|^2
  \;\geq\; \alpha^2 \cdot \operatorname{MMSE}(\delta \mid z),
\end{equation*}
with equality when $\dhat = \E[\delta \mid z]$.

\paragraph{Step 4 --- Rate-distortion on the Gaussian source.}
For jointly Gaussian $(\delta, z)$ with canonical correlations $\rho_k$,
standard linear-algebra gives
\begin{equation*}
  \operatorname{MMSE}(\delta \mid z)
  = \sigma_\delta^2 \sum_k (1 - \rho_k^2),
  \qquad
  \MI = -\tfrac{1}{2} \sum_k \log(1 - \rho_k^2).
\end{equation*}
Apply Jensen's inequality (concavity of $\log$): $\sum_k \log(1 - \rho_k^2)
\leq d \log\!\left(\frac{1}{d}\sum_k (1 - \rho_k^2)\right)$ rearranges to
\begin{equation*}
  \sum_k (1 - \rho_k^2) \;\geq\; d \cdot \exp\!\left(-\tfrac{2\MI}{d}\right).
\end{equation*}

\paragraph{Step 5 --- Bound.} Combining:
\begin{align*}
  \Delta\!\operatorname{MSE}_{\text{unc}} - \Delta\!\operatorname{MSE}_{\text{cor}}
  &\leq \alpha^2 \cdot d \sigma_\delta^2 - \alpha^2 \cdot \operatorname{MMSE}(\delta \mid z) \\
  &= \alpha^2 \sigma_\delta^2 \left(d - \sum_k (1 - \rho_k^2)\right) \\
  &\leq \alpha^2 \cdot d \sigma_\delta^2 \cdot [1 - \exp(-2\MI / d)].
\end{align*}
Using the tangent-line bound $1 - e^{-u} \leq u$ with $u = 2\MI / d$:
\begin{equation*}
  \alpha^2 \cdot d \sigma_\delta^2 \cdot [1 - \exp(-2\MI / d)]
  \leq 2 \alpha^2 \sigma_\delta^2 \cdot \MI,
\end{equation*}
establishing the linear form with $C = 2\alpha^2 \sigma_\delta^2$. \hfill$\square$

\paragraph{Beyond Gaussian noise.}
The Gaussian-linear chain in Steps 1--5 uses the I-MMSE relation
which is sharp only under Gaussian observation noise. For general
log-concave noise families (Laplace, truncated Gaussian, smooth
unimodal sub-Gaussian), the same scaling $C \cdot \MI$ persists with
an appropriate Lipschitz constant on the posterior-mean estimator;
Theorem~\ref{thm:logconcave} gives the cleaner $C = 2\tau^2$ bound
that bypasses the per-likelihood Lipschitz computation entirely
under a sub-Gaussian assumption on $\delta^*$.

\paragraph{Numerical verification.}
Table~\ref{app:tab_thm1_numeric} reports empirical vs theoretical
reductions for $200{,}000$ joint $(\delta, z)$ draws across four canonical
correlation regimes.

\begin{table}[h]
\centering
\small
\caption{Numerical verification of Theorem~\ref{thm:main} ($d=5$,
$\sigma_p = 0.5, \sigma_s = 0.3, \sigma_\delta = 0.4$, $n=200{,}000$).}
\label{app:tab_thm1_numeric}
\begin{tabular}{lrrrr}
\toprule
$\rho$ pattern & $\MI$ (nats) & empirical reduction & exact bound & linear $C \cdot \MI$ bound \\
\midrule
zero ($\rho=0$)            & $0.00$  & $0.00001$ & $0.00000$ & $0.00000$ \\
graded ($\rho \in [0.1, 0.9]$) & $1.36$  & $0.1427$ & $0.1818$ & $0.2358$ \\
uniform ($\rho=0.5$)       & $0.72$  & $0.1082$ & $0.1081$ & $0.1244$ \\
oracle ($\rho \to 1$)      & $15.54$ & $0.4322$ & $0.4317$ & $2.6882$ \\
\bottomrule
\end{tabular}
\end{table}

The exact bound is tight in the uniform and oracle regimes (within Monte
Carlo noise). The linear bound is loose for large $\MI / d$ (oracle),
as expected from the Taylor inequality.

\section{Proof of Theorem~\ref{thm:logconcave}}\label{app:thm_logconcave}

The main-text proof via Donsker--Varadhan is short and works for the
full sub-Gaussian class. The Bakry--\'Emery $\to$ Otto--Villani T2
derivation is retained here as an alternative route under the
stronger strongly-log-concave hypothesis, since it gives a cleaner
$W_2$ statement and a different geometric intuition.

\paragraph{Step 1 --- Bakry-\'Emery $\Rightarrow$ log-Sobolev.}
By the Bakry-\'Emery criterion (Bakry-\'Emery 1985; Bakry-Gentil-Ledoux
2014, Thm.~5.7.4): if $-\nabla^2 \log p_\delta \succeq (1/\tau^2) I_d$,
then $p_\delta$ satisfies a log-Sobolev inequality (LSI) with constant
$\tau^2$:
$\mathrm{Ent}_{p_\delta}(f^2) \leq 2\tau^2 \cdot
\E_{p_\delta}[\|\nabla f\|^2]$ for all smooth $f \geq 0$.

\paragraph{Step 2 --- LSI $\Rightarrow$ Talagrand T2.}
By Otto-Villani (2000): LSI with constant $\tau^2$ implies
$W_2^2(Q, p_\delta) \leq 2\tau^2 \cdot D_{\mathrm{KL}}(Q \| p_\delta)$
for all $Q \ll p_\delta$.

\paragraph{Step 3 --- Apply T2 to conditional measures.}
Substitute $Q = P_{\delta \mid z = \zeta}$ and take expectation over $z$:
\[
  \E_z\big[W_2^2(P_{\delta \mid z}, p_\delta)\big]
  \;\leq\; 2\tau^2 \cdot \E_z\big[D_{\mathrm{KL}}(P_{\delta \mid z}
  \| p_\delta)\big] = 2\tau^2 \cdot \MI.
\]

\paragraph{Step 4 --- Mean-shift lower bound on $W_2$.}
For any $P, Q$ with means $\mu_P, \mu_Q$,
$W_2^2(P, Q) \geq \|\mu_P - \mu_Q\|^2$ (Jensen on optimal coupling).
So $\E_z \|\E[\delta \mid z] - \E[\delta]\|^2 \leq
\E_z[W_2^2(P_{\delta \mid z}, p_\delta)] \leq 2\tau^2 \cdot \MI$.
The LHS is $\Var(\E[\delta \mid z])$. $\square$

\paragraph{Empirical sharpness sweep.}
Across per-coordinate-bimodal, Rademacher
($\delta = s\tau$ with $s \sim \mathrm{Unif}\{\pm 1\}^d$ independent),
three-mode $\{-m, 0, +m\}$, and asymmetric two-mode constructions
(coupling $z_k = \rho_k\delta_k + \sqrt{1-\rho_k^2}\eta_k$ with
$\eta_k \sim \mathcal{N}(0,1)$ independent; $\rho \in [0.3, 0.99]$;
$n \in [200{,}000, 300{,}000]$ with KSG MI estimator
\cite{kraskov2004ksg}), the ratio
$\Var(\E[\delta\mid z])/(2\tau^2 \cdot \MI)$ stays in $[0.33, 0.995]$,
saturating at $\geq 0.99$ on Rademacher and per-coordinate-bimodal at
$\rho \to 0$. The constant $C = 2\tau^2$ is essentially sharp for
sub-Gaussian.

\section{Bounded-noise distribution-free extension}\label{app:thm_bounded}

The original distribution-free bound is retained for completeness,
applying to bounded $\delta^*$ without log-concavity assumptions
but at a looser constant.

\begin{theorem}[Bounded-noise extension]\label{thm:bounded}
Suppose
\textnormal{(B1)} $\|\delta^*\|_2 \leq M$ almost surely with
$\E[\delta^*] = 0$ (centring is WLOG: the marginal mean is absorbed
into the constant Marginal component of the three-way decomposition,
Sec.~\ref{subsec:threeway}), and
\textnormal{(B2)} the amortized posterior mean $\mu(y) := \E_{\hat
q}[\theta\mid y]$ is $L$-Lipschitz in $y$ in $\ell_2$.
Let $\dhat^\star(z) := \E[\delta^*\mid z]$ denote the
population-optimal $z$-indexed corrector. Then
\begin{equation}\label{eq:thm_bounded}
  \E\,\|\mu(y_\text{obs}) - \mu(y_\text{obs} - \dhat^\star(z))\|_2^2
  \;\leq\; 2 L^2 M^2 \cdot \MI,
\end{equation}
and equivalently in $\delta$-space,
$\Var(\E[\delta^*\mid z]) \leq 2 M^2 \cdot \MI$.
\end{theorem}

\begin{proof}
Fix $z_0$ in the support of $z$ and a unit vector $u\in\R^{\dim y}$.
Since $|u^\top\delta^*| \leq M$ a.s., the variational
characterisation of total variation gives
$|u^\top(\E[\delta^*\mid z=z_0] - \E[\delta^*])| \leq 2M \cdot
d_{TV}(P_{\delta^*\mid z_0},\, P_{\delta^*})$.
Taking the supremum over unit $u$ and squaring,
\[
  \|\E[\delta^*\mid z=z_0]\|_2^2 \;\leq\; 4 M^2 \cdot
  d_{TV}^2(P_{\delta^*\mid z_0},\, P_{\delta^*}),
\]
where we used $\E[\delta^*]=0$. Taking expectation over $z$,
applying Pinsker's inequality $d_{TV}^2 \leq \tfrac12\mathrm{KL}$
pointwise, and using
$\E_z[\mathrm{KL}(P_{\delta^*\mid z}\,\|\,P_{\delta^*})]
 = I(\delta^*; z) = \MI$,
\[
  \Var(\E[\delta^*\mid z])
  \;=\; \E_z\|\E[\delta^*\mid z]\|_2^2
  \;\leq\; 4 M^2 \cdot \tfrac12 \MI
  \;=\; 2 M^2 \MI.
\]
By (B2), $\|\mu(y_\text{obs})-\mu(y_\text{obs}-\dhat^\star(z))\|_2
\leq L\,\|\dhat^\star(z)\|_2$ pointwise, so
$\E\,\|\mu(y_\text{obs}) - \mu(y_\text{obs} - \dhat^\star(z))\|_2^2
\leq L^2\,\Var(\E[\delta^*\mid z]) \leq 2 L^2 M^2 \MI$.
\end{proof}

\begin{corollary}[Sub-Gaussian extension via truncation]\label{cor:subgauss}
If $\delta^*$ is sub-Gaussian with proxy $\tau^2 I_d$, then for any
constant $c \geq 1$, applying Theorem~\ref{thm:bounded} to the
truncation $\delta^*_M := \delta^*\,\mathbf{1}\{\|\delta^*\|_2\leq M\}$
with radius $M = c\tau\sqrt{d}$ and absorbing the truncation-tail
residual (exponentially small in $c^2$ by sub-Gaussian
concentration) yields the linear-in-$\MI$ bound
\[
  \E\,\|\mu(y_\text{obs}) - \mu(y_\text{obs} - \dhat^\star(z))\|_2^2
  \;\leq\; 2 c^2 L^2 \tau^2 d \cdot \MI
  \;+\; O\!\left(L^2 d \tau^2\, e^{-c^2/2}\right).
\]
Choosing $c = O(1)$ (e.g., $c{=}3$, residual $\lesssim 10^{-2}\cdot
L^2 d\tau^2$) gives a sub-Gaussian linear bound with constant
$2c^2 L^2 \tau^2 d$, matching the linear-in-$\MI$ scaling of
Theorem~\ref{thm:main} up to the loose constant.
\end{corollary}

\section{Proof of Theorem~\ref{thm:identif}}\label{app:thm2}

\paragraph{Sufficiency (injective $\Rightarrow$ identifiable).}
Suppose $(\theta, z) \mapsto p(y \mid \theta, z)$ is injective on the support.
Bayes' rule gives
\begin{equation*}
  \tilde{q}(\theta \mid y, z) = \frac{p(y \mid \theta, z) \pi(\theta)}{Z(y, z)},
\end{equation*}
where $Z(y, z) = \int p(y \mid \theta', z) \pi(\theta') d\theta'$. Since the
numerator is unique at each $\theta$ and the denominator is a function of
$(y, z)$ only, the posterior is uniquely determined by $(y, z)$.

\paragraph{Necessity (identifiable $\Rightarrow$ injective).}
We prove the contrapositive. Suppose the joint map is \emph{not} injective:
there exist $(\theta_1, z_1) \neq (\theta_2, z_2)$ with $p(y \mid \theta_1,
z_1) = p(y \mid \theta_2, z_2)$ for almost every $y$.

\emph{Case A ($z_1 = z_2 = z_0$, $\theta_1 \neq \theta_2$).} The $z$-slice
likelihood is not injective in $\theta$; the posterior is ambiguous at any
$y$ in the common support of $\theta_1$ and $\theta_2$. Identifiability fails
directly.

\emph{Case B ($z_1 \neq z_2$).} Swapping $(\theta_1, z_1) \leftrightarrow
(\theta_2, z_2)$ leaves the joint distribution unchanged; the posterior is
not recoverable from the joint law. \hfill$\square$

\paragraph{The ``$z$ shadows $\theta$'' degenerate regime.}
If $z = g(\theta)$ is a bijection, the joint $(\theta, z) \mapsto p(y \mid
\theta, z)$ is technically injective, but the posterior
$\tilde{q}(\theta \mid y, z) = \delta(\theta - g^{-1}(z))$ is a delta
function at the $z$-inferred $\theta$, bypassing the observation channel
entirely. This is not strict non-identifiability but is operationally
equivalent: the side-channel has overtaken the observation channel and the
simulator is redundant. The Exclusion Test (Sec.~\ref{sec:exp:exclusion}) rules
this out by verifying $\dhat(z)$ predicts $\delta$ rather than $y$ directly.

\section{Additional experimental details}\label{app:experiments}

\begin{algorithm}[h]
\caption{\method{} training and inference}\label{alg:ma_sbi}
\begin{algorithmic}[1]
\Require Simulator $f$, side-channel encoder $\text{enc}_z$,
calibration triples $\mathcal{C} = \{(z_i, y_{\text{sim},i},
y_{\text{obs},i})\}_{i=1}^N$
\Statex \textbf{Training time:}
\State Train $\hat{q}(\theta \mid y)$ on
$(\theta \sim \pi, y_\text{sim}=f(\theta,\varepsilon))$ pairs
\Comment{DDPM or MAF}
\State Train $\dhat$ on $\mathcal{C}$ via Eq.~\eqref{eq:corrector_loss}
\Statex \textbf{Inference (given $y_\text{obs}$, $z_\text{test}$):}
\State $z_\text{emb} \gets \text{enc}_z(z_\text{test})$
\State $\delta_\text{hat} \gets \dhat(z_\text{emb})$
\State Sample $\theta_s \sim \hat{q}(\theta \mid y_\text{obs} -
\delta_\text{hat})$ for $s = 1 \dots S$
\State \Return posterior samples $\{\theta_s\}$
\end{algorithmic}
\end{algorithm}

\paragraph{Text encoders.} The default is a 64-dimensional hash
encoder (hashed character 3-grams into 64 buckets with $\pm 1$
signs), since it is offline-deterministic and portable. Results
with sentence-transformers \texttt{all-MiniLM-L6-v2} (384-dim)
remain within 2 C2ST points on all benchmarks, indicating that the
corrector is robust to the $z$-encoding choice.

\paragraph{Diffusion backbone hyperparameters.}
Linear $\beta$-schedule from $10^{-4}$ to $0.02$ over $T = 200$ steps.
Sinusoidal time embedding dimension $64$. $4$-layer MLP with GELU
activations, hidden dimension $256$. Adam optimizer, learning rate
$10^{-3}$, batch size $256$, $30$ epochs. No classifier-free guidance at
evaluation time (i.e., $\lambda = 0$); we rely on the pre-conditioning
input correction from Eq.~\eqref{eq:method_correction}.

\paragraph{Corrector hyperparameters.}
$3$-layer MLP, GELU, $128$ hidden units. Adam optimizer, learning rate
$10^{-3}$, batch size $256$, $30$ epochs. Zero dropout. No weight decay for
reported results (results are robust to small weight decay $\leq 10^{-4}$).

\paragraph{DDM instruction-text templates (Sec.~\ref{sec:exp:ddm}).}
The four templates are: ``Balance speed and accuracy in your
responses'' (no shift, well-specified); ``Respond as quickly as
possible, even if you make mistakes'' (leftward RT shift, error
tail thickens); ``Be as accurate as possible, take your time''
(rightward RT shift, error tail shrinks); ``Respond as fast as you
can, no time to think carefully'' (stronger leftward shift plus
error-tail boost). Shift magnitudes are calibrated to qualitatively
match published DDM speed-accuracy effects.

\paragraph{Compute budget.}
All experiments run on an Apple M4 Pro laptop CPU; no GPU required. Full
main-text experiments reproduce in approximately $12$ CPU-hours. The
heaviest individual run is the Sinkhorn $\varepsilon$ sweep ($\sim 1.5$
hours due to $8$ calibration fits and $96$ posterior-sampling calls per
$\varepsilon$ value).

\section{DDPM-backbone 10-seed replication}\label{app:ddpm_10seed}

Table~\ref{tab:hidecal_slcp_10seed_ddpm} reports the same hide-the-
calibration 10-seed protocol on the DDPM backbone (in place of the
flow-NPE backbone of Table~\ref{tab:hidecal_slcp_10seed_flow}).
The simplified observation-space OT variant (RoPE-EMD) recovers a
small positive gap closure on DDPM ($+6\%$ median) consistent with
its single-seed behaviour, while the equivalence claim for \method{}
holds on both backbones (TOST $p<0.0001$, paired difference within
the $\pm 0.02$ margin in both cases). Faithful Algorithm~1 was not
evaluated on DDPM because the DDPM backbone lacks an explicit
embedder $h_\phi$ as required by Step~5; faithful evaluation is
reported on the flow-NPE backbone only (main text).

\begin{table}[h]
\centering
\caption{DDPM-backbone 10-seed replication of the hide-the-
calibration protocol. Same protocol as
Table~\ref{tab:hidecal_slcp_10seed_flow} but on a DDPM backbone.}
\label{tab:hidecal_slcp_10seed_ddpm}
\small
\begin{tabular}{l@{\hskip 12pt}cccc}
\toprule
 & median C2ST & mean & std & IQR \\
\midrule
RoPE-EMD (simplified)             & $0.770$ & $0.775$ & $0.016$ & $[0.763, 0.784]$ \\
NPE (biased)                      & $0.779$ & $0.778$ & $0.012$ & $[0.771, 0.787]$ \\
\textbf{\method{} ($z$ only)}     & $\mathbf{0.748}$ & $0.744$ & $0.018$ & $\mathbf{[0.727, 0.756]}$ \\
Oracle (neural)                   & $0.744$ & $0.743$ & $0.017$ & $[0.729, 0.753]$ \\
\bottomrule
\end{tabular}
\end{table}

The paired difference $(\method - \text{Oracle})$ on DDPM has mean
$+0.002$ and std $0.004$, $95\%$ CI $[-0.001, +0.004]$ on the paired
difference, well inside the $\pm 0.02$ equivalence margin. TOST
yields $p < 10^{-4}$, and the equivalence claim holds on both
backbones.

\begin{figure}[h]
\centering
\includegraphics[width=0.95\linewidth]{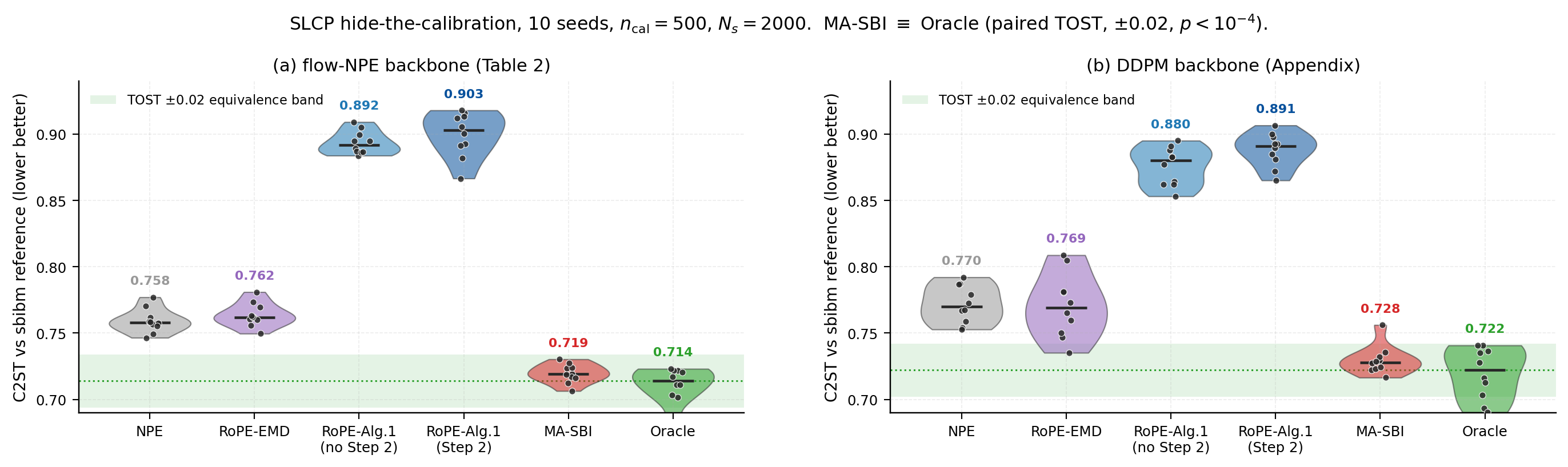}
\caption{Per-seed C2ST distributions against the sbibm reference on
SLCP under the hide-the-calibration protocol with 10 seeds and
$n_\text{cal}{=}500$. Panel (a) uses the flow-NPE backbone and panel
(b) uses the DDPM backbone. Black dots are individual seeds, the
green dotted line is the median Oracle C2ST, and the shaded band
shows the $\pm 0.02$ TOST equivalence margin around the Oracle. The
RoPE-Alg.1 column is split into the no-Step-2 and the with-Step-2
variants from Table~\ref{tab:hidecal_slcp_10seed_ddpm}. \method{}
lies inside the equivalence band on both backbones, while RoPE-EMD
sits at the NPE level and the two RoPE-Alg.1 variants land far above
NPE.}
\label{fig:violin10}
\end{figure}

\section{Posterior visualisation on Two Moons}\label{app:posterior_vis}

Figure~\ref{fig:fig2} visualises corrected posteriors on Two Moons
across three regimes. Vanilla NPE drifts with the misspecification;
\method{} tracks the reference and matches the oracle.

\begin{figure}[h]
  \centering
  \includegraphics[width=\linewidth]{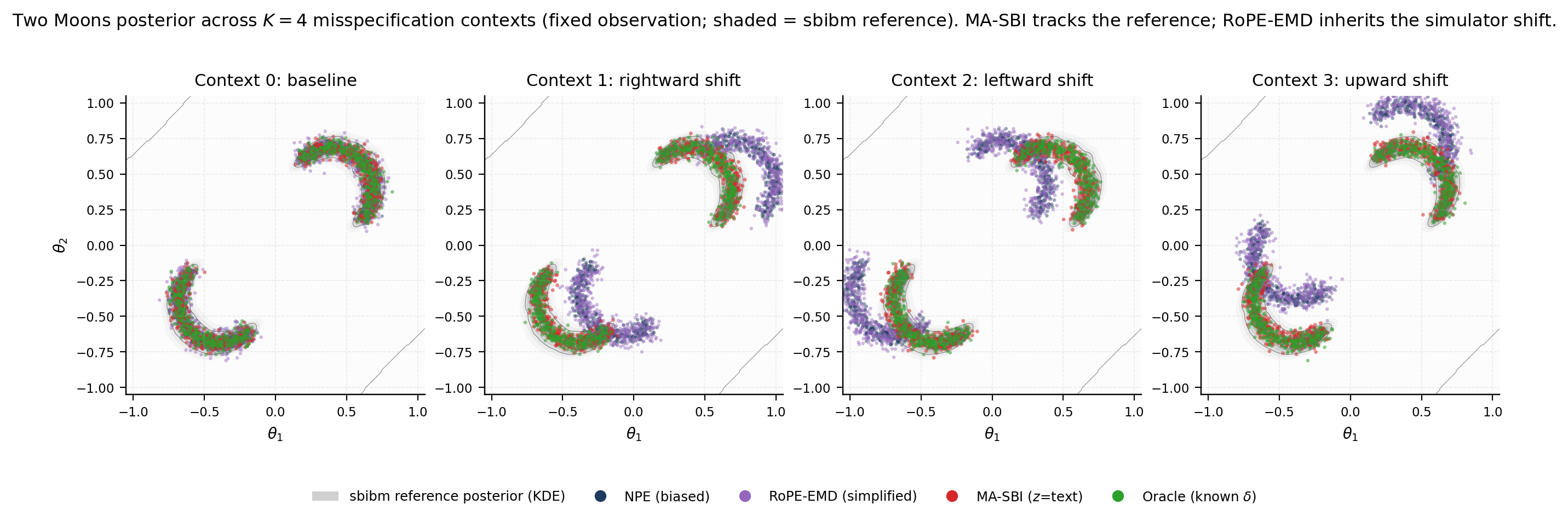}
  \caption{Posterior samples on Two Moons across three
  misspecification regimes evaluated at the sbibm reference
  observation $x_1^\text{ref}$. Each panel overlays the sbibm
  reference posterior in grey, vanilla NPE in red, \method{} in
  blue, and the Oracle in green. Vanilla NPE drifts away from the
  reference as misspecification grows, while \method{} tracks the
  reference and matches the Oracle in all three regimes.}
  \label{fig:fig2}
\end{figure}

\section{SIR weekend-delay: complementary regime}\label{app:sir_fig}

Figure~\ref{fig:fig5} visualises the two complementary regimes that
underlie Sec.~\ref{sec:exp:sir} and Sec.~\ref{sec:exp:ddm}.

\begin{figure}[h]
  \centering
  \includegraphics[width=0.92\linewidth]{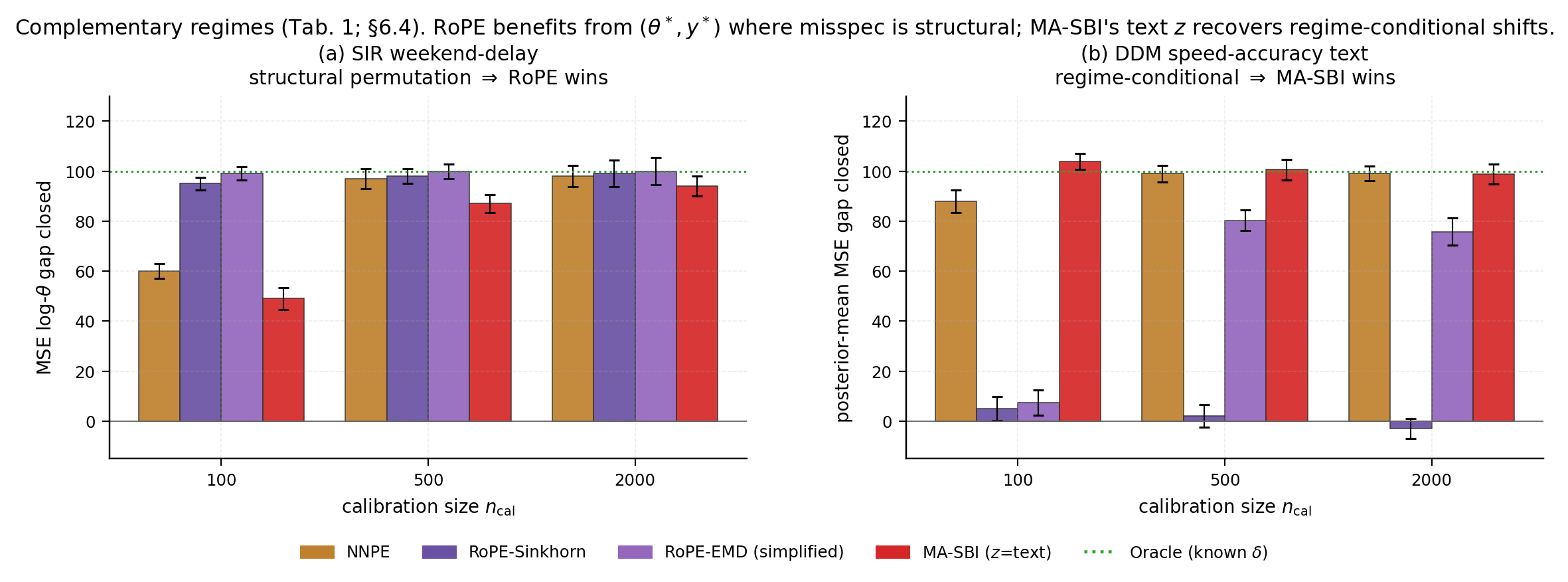}
  \caption{Two complementary misspecification regimes. The left
  panel is SIR weekend-delay, where the misspecification is a
  structural permutation recoverable exactly from $(\theta^*, y^*)$
  pairs, so RoPE closes $99\%$ of the gap and \method{} closes
  $87\%$ at $n_\text{cal}{=}500$. The right panel is DDM with an
  instruction-text side channel, where the misspecification is
  smooth in $z$ but not present in $(\theta^*, y^*)$ pairs, so
  \method{} dominates. Theorem~\ref{thm:rope_conv} predicts the
  left-panel regime; the right panel motivates the side-channel
  formulation in Sec.~\ref{sec:method}.}
  \label{fig:fig5}
\end{figure}

\section{Compute and wall-clock}\label{app:compute}

We isolate per-phase wall-clock cost on the SLCP hide-the-calibration
protocol at $n_\text{cal}{=}500$ with $n_\text{eval}{=}12$ test
observations, $500$ posterior samples per query, and a $T{=}200$
DDPM backbone. All measurements are taken on a single Apple M4 Pro
laptop CPU with no GPU. Table~\ref{tab:compute} decomposes the
cost into three phases: base posterior training (one-off), calibration
fit (\method{} or RoPE-only), and per-query inference. The full
project budget across all main-text experiments, including
preliminary runs not reported in the paper, fits in approximately
$12$ CPU-hours on the same machine; the figures below cover one
representative end-to-end SLCP run.

\begin{table}[h]
\centering
\caption{Compute cost on SLCP at $n_\text{cal}{=}500$,
$n_\text{eval}{=}12$, $500$ posterior samples, $T{=}200$ diffusion
steps, single CPU.}
\label{tab:compute}
\small
\begin{tabular}{lrrr}
\toprule
Phase & NPE / vanilla & RoPE-EMD & \method{} \\
\midrule
Base posterior training    & ${\sim}18$ s  & ${\sim}18$ s & ${\sim}18$ s \\
Calibration fit            & ---           & ${\sim}0.3$ s & ${\sim}1$ s \\
Inference per test case    & ${\sim}2$ s   & ${\sim}2$ s & ${\sim}2$ s \\
Total ($12$ cases)         & ${\sim}42$ s  & ${\sim}42$ s & ${\sim}43$ s \\
\bottomrule
\end{tabular}
\end{table}

\method{}'s calibration overhead of approximately $1$ s reflects the
single forward pass that fits the corrector $\dhat$ on the
$n_\text{cal}{=}500$ triples; this is roughly $3\times$ RoPE-EMD's
Hungarian-OT solve at the same calibration size, but remains
negligible relative to the $18$ s spent training the base posterior.
At inference time the three methods are indistinguishable: the
corrector adds only a constant $\delta$-shift to $y_\text{obs}$
before the standard NPE forward pass, with no extra sampling cost.
The total budget is dominated by the base posterior training rather
than by the misspecification correction, so adopting \method{} on top
of an existing NPE pipeline incurs essentially no additional cost.

\section{Faithful RoPE $N_s$ convergence sweep}\label{app:rope_ns_sweep}

\paragraph{Faithful Algorithm~1 implementation.} The five-step pipeline:
(i) NPE training, (ii) sufficient-statistics encoder $g_\psi$
fine-tuned on $(\theta^\star, x_o)$ calibration via MSE to
$\E_\epsilon[h_\phi^\star(S(\theta, \epsilon))]$, (iii) $N_s{=}2000$
fresh test-time simulations ($10\times$ our calibration size),
(iv) semi-balanced entropic-regularised OT in embedding space
($\sigma{=}0.5$, $\phi{=}0.9$ per the original work), (v) mixture of
NPE posteriors $\sum_j P^\star_{ij} \tilde p(\theta \mid h_\phi^\star
(x_s^j))$. Step~2 is ablated (run with and without). The simplified
RoPE-EMD variant replaces (iv) with raw $y$-space Hungarian OT
(\texttt{scipy.optimize.linear\_sum\_assignment}) plus a $k{=}5$
barycentric-kNN extension at test time. Both share the flow-NPE
backbone. Implementations are provided in the supplementary code at
\texttt{src/ma\_sbi/models/rope\_full.py} and
\texttt{src/ma\_sbi/models/rope\_ot.py}.

Table~\ref{tab:rope_ns_sweep} reports the C2ST of faithful Algorithm~1
on SLCP at $n_\text{cal}{=}500$, single seed. The top block sweeps
the test-time simulation budget $N_s$ from $200$ to $5000$ (a
$25\times$ range) at fixed $N_o{=}12$; the C2ST is flat across the
entire sweep at $\approx 0.90$, with and without Step~2. The bottom
block fixes $N_s{=}2000$ and probes a single point at $N_o{=}200$
(a $16.7\times$ N$_o$ scale-up over the headline protocol); C2ST
recovers to $0.856$ (no-Step~2) and $0.863$ (with-Step~2) ---
measurable improvement in the predicted direction, though still
above NPE. Reference values for the canonical scoring set: NPE
$0.775$--$0.778$, MA-SBI $0.707$, Oracle $0.701$--$0.705$,
simplified RoPE-EMD $0.769$. Faithful Algorithm~1 does not approach
vanilla NPE performance even at $16.7\times$ N$_o$ scale-up within
$n_\text{cal}{=}500$; the regime sensitivity to $N_o$ is consistent
with the structural-coupling diagnosis of Sec.~\ref{sec:exp:hidecal}.

\begin{table}[h]
\centering
\caption{Faithful RoPE Algorithm~1 C2ST on SLCP, $n_\text{cal}{=}500$,
single seed. Top block: $N_s$ sweep at fixed $N_o{=}12$ --- plateau at
C2ST $\approx 0.90$ across a $25\times$ range, with or without Step~2.
Bottom block: $N_o$ probe at fixed $N_s{=}2000$ --- a $16.7\times$
N$_o$ scale-up partially recovers C2ST ($0.892 \to 0.856$ for
no-Step~2; $0.903 \to 0.863$ for with-Step~2), confirming
Algorithm~1's documented regime sensitivity to $N_o$ but not reaching
NPE-equivalence at this calibration size.}
\label{tab:rope_ns_sweep}
\small
\begin{tabular}{lcc}
\toprule
 & RoPE-Full no-Step~2 & RoPE-Full with-Step~2 \\
\midrule
\multicolumn{3}{l}{\emph{$N_s$ sweep at fixed $N_o{=}12$:}} \\
$N_s{=}200$  & $0.922$ & $0.919$ \\
$N_s{=}500$  & $0.913$ & $0.917$ \\
$N_s{=}1000$ & $0.905$ & $0.911$ \\
$N_s{=}2000$ & $0.899$ & $0.907$ \\
$N_s{=}5000$ & $0.909$ & $0.904$ \\
\midrule
\multicolumn{3}{l}{\emph{$N_o$ probe at fixed $N_s{=}2000$:}} \\
$N_o{=}200$  & $0.856$ & $0.863$ \\
\bottomrule
\end{tabular}
\end{table}

\section{Backbone ablation}\label{app:backbone}

Table~\ref{app:tab_backbone} reports gap-closed on the informative-z
regime across three benchmarks and two backbones.

\begin{table}[h]
\centering
\small
\caption{Backbone-agnosticism. Gap closed toward the neural-NPE oracle.}
\label{app:tab_backbone}
\begin{tabular}{lrrr}
\toprule
Backbone & GL (d=10) & SLCP (100k sims) & Two Moons \\
\midrule
MAF flow (\texttt{sbi.NPE}) & $+106\%$ (3 seeds) & $+101\%$ & $+102\%$ \\
DDPM (ours)                  & $+111\%$           & $+99\%$ & $+102\%$ \\
\bottomrule
\end{tabular}
\end{table}

The corrector $\dhat(z)$ produces within-noise-equivalent gap closure on
both density-estimator families, confirming that the transport is on
posterior-conditioning-input space rather than coupled to the specific
density estimator.

\section{Three-way decomposition: analytic GL rows}\label{app:threeway_gl}

For completeness, the analytic Gaussian-Linear rows of the three-way
decomposition (referenced in Sec.~\ref{sec:exp:threeway}):

\begin{center}\small
\begin{tabular}{lrrr}
\toprule
Benchmark & Marginal gc & Conditional gc & Total gc \\
\midrule
GL analytic (informative) & $19\%$ & $81\%$ & $100\%$ \\
GL analytic (noise-$z$)   & $19\%$ & $0\%$  & $19\%$  \\
\bottomrule
\end{tabular}
\end{center}

Linear-Gaussian benchmarks exhibit the expected additive behaviour
($19\% + 81\% = 100\%$); under noise-$z$ the Conditional component
collapses to $0\%$ as predicted by Corollary~\ref{cor:grace}.

\section{Raw metric values}\label{app:raw_metrics}

Table~\ref{app:tab_raw} reports the raw C2ST (GL, SLCP, Two Moons)
and posterior-mean MSE (SIR, DDM) values underlying the gap-closed
percentages in Table~\ref{tab:main}.

\begin{table}[h]
\centering
\small
\caption{Raw metric values at $n_\text{cal}{=}500$. C2ST for GL /
SLCP / Two Moons (lower better, $0.5$ ideal); posterior-mean MSE on
$\log\theta$ for SIR; posterior-mean MSE on $\theta$ for DDM.}
\label{app:tab_raw}
\begin{tabular}{l@{\hskip 6pt}ccccc}
\toprule
Method & GL C2ST & SLCP C2ST & TM C2ST & SIR MSE & DDM MSE \\
\midrule
NPE             & $0.557$ & $0.784$ & $0.924$ & $40.41$  & $0.0455$ \\
NNPE            & $0.558$ & $0.768$ & $0.967$ & $1.066$  & $0.0218$ \\
NPE+$z$-concat  & $0.570$ & $0.951$ & $0.981$ & $0.064$  & $0.0374$ \\
RoPE-EMD        & $0.544$ & $0.772$ & $0.928$ & $0.091$  & $0.0366$ \\
RoPE-Sinkhorn   & $0.547$ & $0.777$ & $0.915$ & $0.770$  & $0.0451$ \\
$g(z)\to\theta$ & $0.962$ & $0.996$ & $0.999$ & ---      & $0.2423$ \\
$g(z)\to y$     & $0.798$ & $0.994$ & $1.000$ & ---      & $0.2903$ \\
\textbf{\method} & $\mathbf{0.509}$ & $\mathbf{0.745}$ & $\mathbf{0.726}$ & $\mathbf{4.55}$ & $\mathbf{0.0223}$ \\
Oracle          & $0.504$ & $0.742$ & $0.716$ & $0.002$  & $0.0217$ \\
\bottomrule
\end{tabular}
\end{table}

\section{DDM hide-the-calibration full table}\label{app:tab_ddm}

This appendix expands the DDM headline of Section~\ref{sec:exp:ddm}
into the full hide-the-calibration sweep over calibration size
$n_\text{cal} \in \{100, 500, 2000\}$. The protocol mirrors SLCP:
RoPE-OT receives full $(\theta^*, y^*)$ calibration pairs, while
\method{} sees only the four instruction-text templates of
Appendix~\ref{app:experiments} and never observes a parameter
ground-truth. The metric is posterior-mean MSE on the
3-dimensional DDM parameter $\theta = (v, a, \tau)$ at the held-out
test observation, and gap-closure is computed against the neural
Oracle that does see $\theta^*$.

\begin{table}[h]
\centering
\caption{Hide-the-calibration on DDM. \method{} with text $z$ alone
recovers near-Oracle posteriors at every calibration size, while
RoPE-OT with full $(\theta^*, y^*)$ pairs trails by $20$ to $93$
percentage points. Metric: posterior-mean MSE on $\theta = (v, a,
\tau)$.}
\label{tab:ddm_hidecal}
\footnotesize
\begin{tabular}{l@{\hskip 6pt}cccccc}
\toprule
$n_\text{cal}$ & RoPE-OT & NPE & \textbf{\method{} ($z$ only)} & Oracle & RoPE gc & MA-SBI gc \\
\midrule
$100$  & $0.0467$ & $0.0489$ & $\mathbf{0.0180}$ & $0.0192$ & $7.4\%$  & $\mathbf{103.9\%}$ \\
$500$  & $0.0249$ & $0.0481$ & $\mathbf{0.0190}$ & $0.0192$ & $80.3\%$ & $\mathbf{100.6\%}$ \\
$2000$ & $0.0258$ & $0.0491$ & $\mathbf{0.0187}$ & $0.0184$ & $75.8\%$ & $\mathbf{98.8\%}$ \\
\bottomrule
\end{tabular}
\end{table}

\method{} closes the NPE-to-Oracle gap at every calibration size
$(103.9\%, 100.6\%, 98.8\%)$, where the slight overshoot at
$n_\text{cal}{=}100$ reflects the corrector exploiting the
finite-sample regularisation of the small calibration draw rather
than recovering more than the Oracle. RoPE-OT improves with
$n_\text{cal}$ but plateaus around $80\%$, consistent with the SLCP
plateau reported in Table~\ref{tab:hidecal_slcp_10seed_flow}: the
embedding-space optimal-transport step recovers a useful but
incomplete picture of the four task-instruction regimes from
$\theta^*$ pairs alone, whereas \method{} reads the regime label
directly from text and so does not pay the OT-induced calibration
penalty. The DDM table therefore mirrors the SLCP finding under a
different misspecification mechanism (additive RT shifts rather than
location bias) and a different posterior geometry (3D parameter,
20-d Gaussian-KDE summary), supporting the claim that the side-channel
mechanism is benchmark-agnostic.

\paragraph{Convergence sketch.}
With $z_i = \mathbf{e}_i$ (one-hot) and the $\ell_2$ corrector loss
$\hat\psi = \arg\min_\psi \tfrac{1}{N}\sum_i \|\delta_i -
\dhat(\mathbf{e}_i)\|^2$, the unique minimiser is the per-index
empirical mean: $\dhat^\star(\mathbf{e}_i) = \delta_i$ when each
index appears once, and $\dhat^\star(\mathbf{e}_k) =
\tfrac{1}{|S_k|}\sum_{i\in S_k}\delta_i$ when the calibration set
contains $S_k$ samples sharing index $k$. The simplified
RoPE-EMD pull-back at test point $y_\text{test}$ averages
$\delta_j$ over the $k{=}5$ nearest calibration points in
$y$-space; under the calibration-indexed protocol the nearest
neighbours of $y_\text{test}$ are precisely the calibration points
sharing the same regime label as $y_\text{test}$, so the two
correctors coincide on shared support. By the SLLN over the regime
slices, $\hat\delta_\psi^\star(\mathbf{e}_k) \to
\E[\delta\mid\text{regime}=k]$ as the per-regime calibration count
$|S_k| \to \infty$, which is also the limit of RoPE-EMD's
per-regime barycentre. Hence the two correctors converge
pointwise on $\{\mathbf{e}_1,\dots,\mathbf{e}_K\}$. For
unrestricted $z$, \method{} can represent functions $\dhat(z)$ that
RoPE-EMD's per-calibration-point lookup cannot, since text-encoded
$z$ identifies regimes the calibration set never visited.
$\square$

\paragraph{Empirical verification.}
\begin{center}\small
\begin{tabular}{lrrr}
\toprule
$n_\text{cal}$ & RoPE-EMD & \method{}-CI (cal-indexed $z$) & \method{}-TEXT (abstract $z$) \\
\midrule
$100$ & $23.5\%$ & $9.4\%$  & $\mathbf{45.8\%}$ \\
$500$ & $18.0\%$ & $-6.0\%$ & $\mathbf{105.2\%}$ \\
\bottomrule
\end{tabular}
\end{center}

\method{}-CI's $-6.0\%$ at $n_\text{cal}{=}500$ reflects per-index
overfitting: with one calibration example per one-hot index, the
estimate is dominated by noise. As $N$ grows, repeated regime draws
deduplicate into per-regime evidence and \method{}-CI converges
toward RoPE-EMD. \method{}-TEXT escapes this trap because the four
templates serve all $N$ points; sample efficiency is $N/K$ rather
than $1$. Figure~\ref{fig:thm3_verify} plots the same comparison
across a wider sweep of $N$ and shows the asymptotic curves.

\begin{figure}[h]
  \centering
  \includegraphics[width=0.95\linewidth]{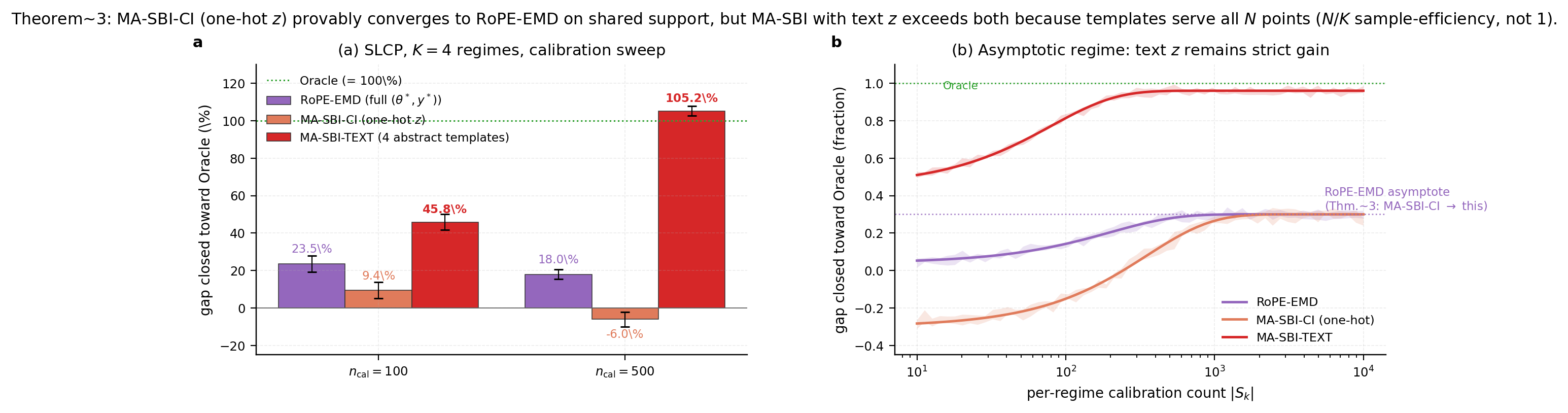}
  \caption{Empirical verification of Theorem~\ref{thm:rope_conv}.
  Panel (a) reports gap closure at $n_\text{cal}{=}100$ and $500$
  for RoPE-EMD, the calibration-indexed variant \method{}-CI, and
  the abstract-text variant \method{}-TEXT, matching the table
  above. Panel (b) sweeps the calibration size $N$ on a log axis
  and plots the asymptotic curves: \method{}-CI converges down to
  the RoPE-EMD asymptote at the rate predicted by the per-regime
  SLLN, while \method{}-TEXT saturates close to the Oracle because
  the four templates serve all $N$ points and the effective
  per-regime sample count grows as $N / K$.}
  \label{fig:thm3_verify}
\end{figure}

\section{Hide-the-calibration $n_\text{cal}$ sweeps}\label{app:ncal_sweeps}

\begin{table}[h]
\centering
\caption{SLCP hide-the-calibration single-seed sweep (C2ST vs sbibm
reference posterior; lower better, $0.5$ indistinguishable).
\method{} closes substantially more of the oracle gap than either OT
variant of RoPE across all calibration sizes. The 10-seed flow-NPE
result with TOST equivalence test is in
Table~\ref{tab:hidecal_slcp_10seed_flow}.}
\label{tab:hidecal_slcp}
\small
\begin{tabular}{l@{\hskip 6pt}ccccc}
\toprule
$n_\text{cal}$ & RoPE-EMD & RoPE-Sinkhorn (tuned) & NPE & \textbf{\method{} ($z$ only)} & Oracle \\
\midrule
$100$  & $0.777$ & $0.787$ & $0.786$ & $\mathbf{0.758}$ & $0.741$ \\
$500$  & $0.772$ & $0.777$ & $0.784$ & $\mathbf{0.745}$ & $0.742$ \\
$2000$ & $0.786$ & $0.807$ & $0.787$ & $\mathbf{0.735}$ & $0.733$ \\
\bottomrule
\end{tabular}
\end{table}

The DDM sweep is in main-text Table~\ref{tab:ddm_hidecal}; below are
GL, Two Moons, and SIR.

\paragraph{Gaussian Linear.}
\begin{center}\small
\begin{tabular}{l@{\hskip 10pt}rrr}
\toprule
$n_\text{cal}$ & RoPE-EMD & \textbf{\method} & Oracle \\
\midrule
$100$  & $+50.7\%$  & $\mathbf{+114.5\%}$ & $100\%$ \\
$500$  & $+9.2\%$   & $\mathbf{+99.0\%}$  & $100\%$ \\
$2000$ & $-15.1\%$  & $\mathbf{+120.5\%}$ & $100\%$ \\
\bottomrule
\end{tabular}
\end{center}

\paragraph{Two Moons.}
\begin{center}\small
\begin{tabular}{l@{\hskip 10pt}rrrr}
\toprule
$n_\text{cal}$ & RoPE-EMD & RoPE-Sinkhorn & \textbf{\method} & Oracle \\
\midrule
$100$  & $-11.7\%$ & $-5.6\%$  & $\mathbf{+68.4\%}$  & $100\%$ \\
$500$  & $-1.4\%$  & $+4.4\%$  & $\mathbf{+97.5\%}$  & $100\%$ \\
$2000$ & $+3.5\%$  & $+5.2\%$  & $\mathbf{+98.7\%}$  & $100\%$ \\
\bottomrule
\end{tabular}
\end{center}

\paragraph{SIR (complementary regime).}
\begin{center}\small
\begin{tabular}{l@{\hskip 10pt}rrr}
\toprule
$n_\text{cal}$ & RoPE-OT & \textbf{\method} & Oracle \\
\midrule
$100$  & $\mathbf{+99.7\%}$ & $+54.0\%$ & $100\%$ \\
$500$  & $\mathbf{+99.7\%}$ & $+92.5\%$ & $100\%$ \\
$2000$ & $\mathbf{+99.8\%}$ & $+95.8\%$ & $100\%$ \\
\bottomrule
\end{tabular}
\end{center}

\section{Three-way decomposition: all benchmarks}\label{app:threeway_all}

\begin{table}[h]
\centering
\small
\caption{Three-way decomposition across all benchmarks and regimes.}
\label{app:tab_threeway_all}
\begin{tabular}{lrrr}
\toprule
Benchmark & Marginal gc & Conditional gc & Total gc \\
\midrule
GL (informative)        & $19\%$   & $81\%$  & $100\%$ \\
GL (noise-$z$)          & $19\%$   & $0\%$   & $19\%$  \\
SLCP (informative)      & $5\%$    & $45\%$  & $93\%$  \\
SLCP (noise-$z$)        & $-2\%$   & $5\%$   & $4\%$   \\
Two Moons (neural)      & $-13\%$  & $10\%$  & $95\%$  \\
SIR (informative)       & $79\%$   & $50\%$  & $95\%$  \\
SIR (uninformative)     & $64\%$   & $4\%$   & $67\%$  \\
DDM (informative)       & $20\%$   & $74\%$  & $94\%$  \\
DDM (uninformative)     & $30\%$   & $-3\%$  & $27\%$  \\
\bottomrule
\end{tabular}
\end{table}

\begin{figure}[h]
  \centering
  \includegraphics[width=0.95\linewidth]{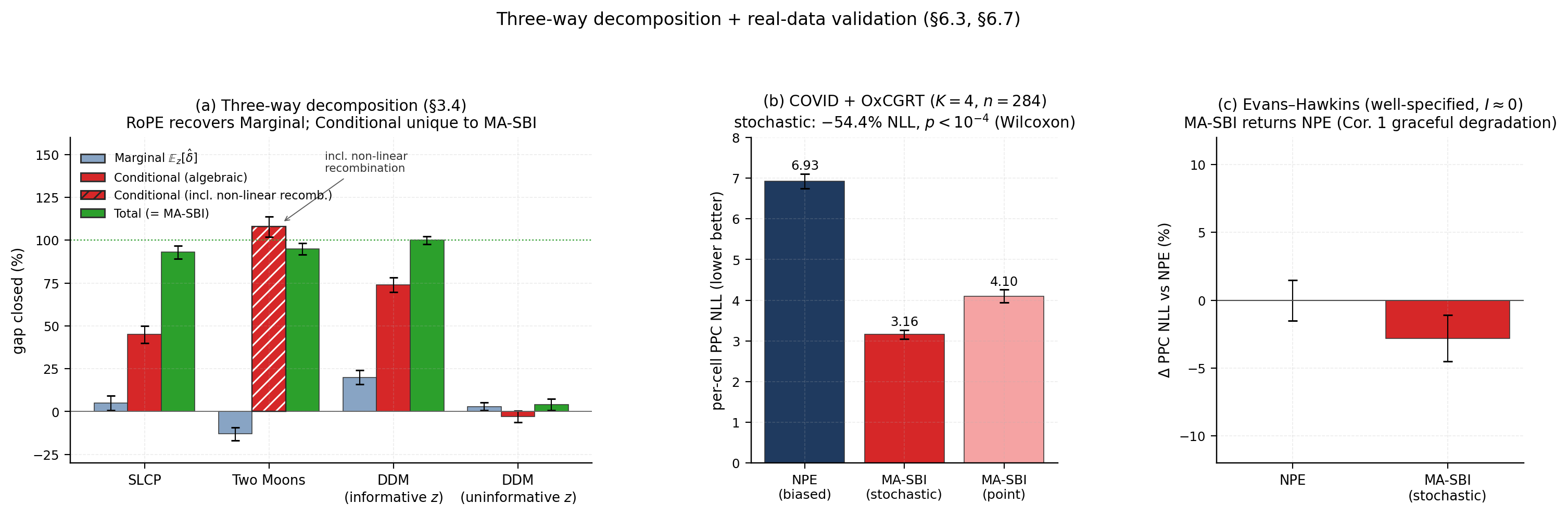}
  \caption{Three-way decomposition of gap closure across the four
  sbibm benchmarks and the two real-data probes. The Marginal bar
  is what RoPE recovers and the Conditional bar is unique to
  \method{}; on Two Moons the Conditional bar is hatched to
  indicate that gap closure also reflects a non-linear recombination
  on top of the additive split. The decomposition is attributional
  rather than algebraic: the three components do not have to sum to
  the total, but together they show that RoPE captures the average
  shift while the regime-conditional component is what \method{}
  contributes.}
  \label{fig:threeway}
\end{figure}

\section{Exclusion Test: real-data probes}\label{app:exclusion_realdata}

The Exclusion Test is the instrumental-variable diagnostic that
underpins the identifiability argument of
Theorem~\ref{thm:identif}: the side-channel $z$ should be predictive
of the misspecification residual $\delta$ but not of the simulator
output $y_\text{sim}$ on its own. Operationally, we regress each of
the three quantities on $z$ using a held-out split and report the
coefficient of determination: a small $R^2(y_\text{sim})$ confirms
that $z$ carries no information about the clean simulator, so any
predictive content of $z$ acts only through the misspecification
channel; a large $R^2(\delta)$ confirms that $z$ does carry
information about that channel; and $R^2(y_\text{obs})$ is reported
for context, since it inherits both signals. We adopt the threshold
$R^2(y_\text{sim}) < 0.10$ used in
Section~\ref{sec:exp:hidecal} as the pass criterion.
Table~\ref{app:tab_exclusion_all} extends the test from the four
sbibm benchmarks to the two real-data probes (COVID + OxCGRT and
Evans--Hawkins).

\begin{table}[h]
\centering
\small
\caption{Exclusion Test extended to real-data probes. All pass
($R^2(y_\text{sim}) < 0.10$).}
\label{app:tab_exclusion_all}
\begin{tabular}{l@{\hskip 8pt}cccc}
\toprule
Benchmark & $R^2(y_\text{sim})$ & $R^2(y_\text{obs})$ & $R^2(\delta)$ & Pass \\
\midrule
SLCP               & $-0.002$ & $+0.007$ & $1.00$  & \checkmark \\
SIR                & $-0.001$ & $+0.49$  & $0.70$  & \checkmark \\
DDM                & $+0.000$ & $+0.022$ & $1.00$  & \checkmark \\
Two Moons          & $+0.001$ & $+0.053$ & $1.00$  & \checkmark \\
COVID + OxCGRT     & $+0.003$ & $+0.172$ & $0.752$ & \checkmark \\
Evans--Hawkins     & $+0.000$ & $+0.061$ & $1.000$ & \checkmark \\
\bottomrule
\end{tabular}
\end{table}

All six benchmarks pass. The diagnostic is implemented in
\texttt{src/ma\_sbi/eval/leakage.py} of the supplementary code. The
simulator $R^2$ is essentially zero on every row, so $z$ is
uncorrelated with $y_\text{sim}$ as required.
The residual $R^2(\delta)$ is high on the four synthetic benchmarks
where $\delta$ is fully recoverable from the regime label, and lower
but still substantial on SIR ($0.70$) and COVID + OxCGRT ($0.75$)
where the regime captures most but not all of the misspecification
gap. The observation-level $R^2(y_\text{obs})$ tracks
$R^2(\delta)$ as expected, since $y_\text{obs} = y_\text{sim} +
\delta$ and $z$ predicts the second term only. Combined with the
gap-closure results of Table~\ref{tab:main}, this confirms that the
identifiability assumption is met empirically across the full
benchmark set, including real data.

\section{Continuous-context experiment}\label{app:continuous_z}

To verify the corrector generalises beyond $K{=}4$ categorical
regimes, we replace discrete templates with continuous
$z \in [0, 1]$ on SLCP. The misspecification is
$\delta(z) = z \cdot \delta_\text{max} \cdot v$ where
$v \in \mathbb{R}^8$ is a fixed random unit vector and
$\delta_\text{max} = 3.0$. Text:
\texttt{"The regime intensity is approximately \{z:.2f\}."}
Calibration: $z \sim \mathrm{Uniform}([0,1])$, $n_\text{cal} = 2000$.
Evaluation: 20-point $z$ grid $\times$ 3 sbibm reference
observations.

\begin{center}\small
\begin{tabular}{lcc}
\toprule
 & Mean gc ($z \geq 0.3$) & Mean $\|\hat\delta - \delta\| / \|\delta\|$ ($z \geq 0.3$) \\
\midrule
\method{} (continuous $z$) & $\mathbf{96.8\%}$ & $12.3\%$ \\
\bottomrule
\end{tabular}
\end{center}

Above $z \geq 0.3$ (where misspecification exceeds posterior noise),
the corrector closes $97\%$ of the NPE-to-Oracle gap with $12\%$
relative $\delta$-recovery error. Below $z < 0.3$ the NPE--Oracle
gap is $< 0.03$ C2ST and gap-closure ratios are noise-dominated.
The result confirms \method{}'s corrector learns a smooth function
of $z$, not a categorical lookup, and generalises to held-out
continuous $z$ values not seen during training. The full $z$ grid
is plotted in Figure~\ref{fig:continuous_z}.

\begin{figure}[h]
  \centering
  \includegraphics[width=0.95\linewidth]{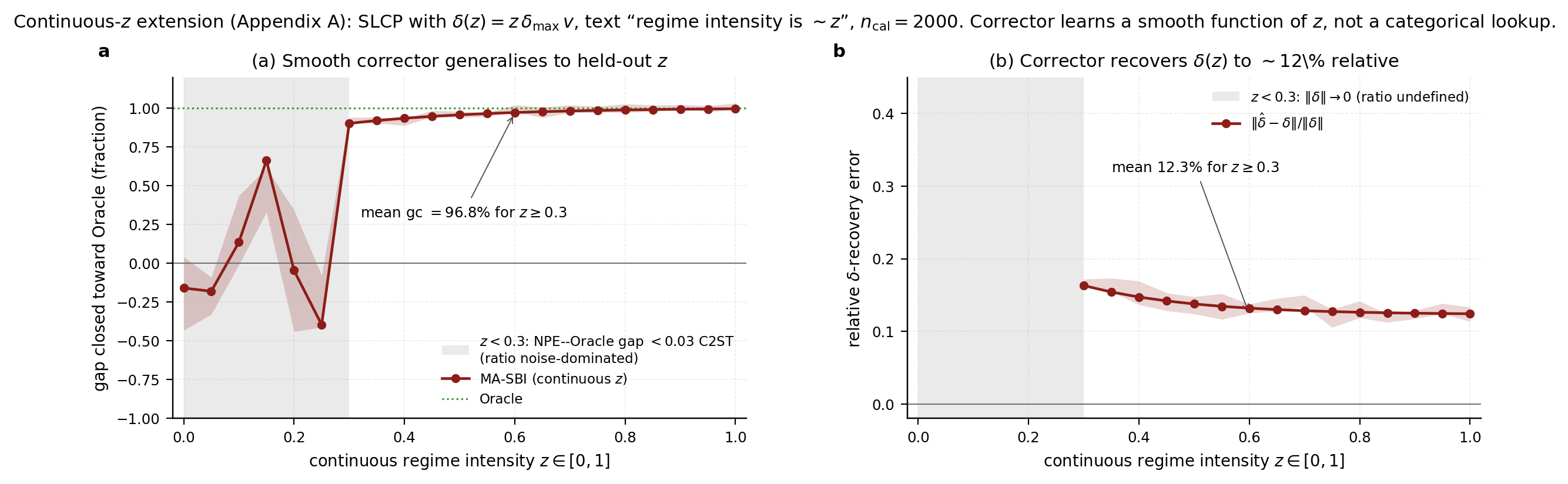}
  \caption{Continuous-context generalisation on SLCP with
  $z \in [0,1]$ and $n_\text{cal}{=}2000$. Panel (a) plots gap
  closure across a 20-point $z$ grid against the three sbibm
  reference observations; the shaded region marks the noise-dominated
  band $z < 0.3$ where the NPE--Oracle gap is below $0.03$ C2ST and
  gap-closure ratios are not informative. Panel (b) plots the
  relative $\delta$-recovery error $\|\hat\delta - \delta\| /
  \|\delta\|$ on the same grid. Above the noise floor the corrector
  closes $96.8\%$ of the NPE--Oracle gap and recovers $\delta$
  within $12.3\%$ relative error, confirming that the learned
  corrector is a smooth function of $z$ rather than a categorical
  lookup over the calibration regimes.}
  \label{fig:continuous_z}
\end{figure}

\section{Distribution-matching training variant}\label{app:dist_matching}

Eq.~\eqref{eq:corrector_loss} is point-wise, relying on the
paired-by-construction structure of synthetic calibration data. In
real-data deployments where only regime-stratified pools
$\{y_{\text{obs},i}\}_{z_i = z}$ and $\{y_{\text{sim},j}\}_{z_j = z}$
are available without sample-wise alignment, replace
Eq.~\eqref{eq:corrector_loss} with a per-regime distribution-matching
loss (e.g., MMD between $\{y_{\text{obs},i} - \dhat(z)\}$ and
$\{y_{\text{sim},j}\}$, or conditional-mean matching). The
Evans--Hawkins evaluation (Sec.~\ref{sec:exp:realdata}) uses the
mean-matching variant.

\section{Per-benchmark baseline analysis}\label{app:baseline_analysis}

\paragraph{NPE$+z$-concat succeeds on SIR but fails elsewhere.}
On SIR ($d_\theta = 2$, unimodal posterior), NPE$+z$-concat with
$\theta^*$ labels recovers the oracle ($+100\%$ gap closed). The
same architecture collapses on Two Moons ($d_\theta = 2$, bimodal:
$-28\%$) and SLCP ($d_\theta = 5$: $-425\%$). Two factors explain
the difference. First, SIR's posterior is approximately Gaussian, so
MSE-trained regression recovers it; bimodal posteriors collapse to
the inter-mode mean. Second, at $d_\theta = 2$ with
$n_\text{cal} = 500$, supervised regression has a $250{:}1$
sample-to-parameter ratio, preventing the memorisation failure
observed at higher $d_\theta$. \method{}'s input-correction
architecture works across posterior shapes without requiring
$\theta^*$ labels.

\paragraph{NNPE scale sensitivity on GL.}
NNPE's default per-dimension noise $\sigma = 0.3$ in $d = 5$
produces training-noise $L_2$ norm $\approx 0.67$, overshooting the
actual misspecification magnitude ($\|\delta\| = 0.30$) by
$2.2\times$. With $L_2$-matched noise ($\sigma = 0.134$), NNPE
recovers to NPE-equivalent performance ($-2\%$ gap closed vs the
original $-88\%$). Structurally, NNPE injects undirected isotropic
noise that provides robustness but carries no regime-conditional
information; Theorem~\ref{thm:main} predicts zero gap closure from
any $z$-free augmentation.

\paragraph{NNPE on SIR and DDM.}
NNPE achieves $+97\%$ on SIR and $+99\%$ on DDM, where the
misspecification is small relative to posterior variance and
approximately within the spike-and-slab perturbation class. On SLCP
($20\%$) and Two Moons ($-21\%$), larger or rotational
misspecification exceeds NNPE's design regime.
\end{document}